\PassOptionsToPackage{table}{xcolor}
\documentclass{article} 
\usepackage{iclr2027_conference,times}

\usepackage{amsmath,amsfonts,bm}

\def\eqref#1{equation~\ref{#1}}

\def\1{\bm{1}}

\DeclareMathAlphabet{\mathsfit}{\encodingdefault}{\sfdefault}{m}{sl}
\SetMathAlphabet{\mathsfit}{bold}{\encodingdefault}{\sfdefault}{bx}{n}

\usepackage{times}
\usepackage{latexsym}

\usepackage[T1]{fontenc}

\usepackage{hyperref}
\usepackage{url}

\usepackage{microtype}

\usepackage{inconsolata}

\usepackage{graphicx}

\usepackage{amsmath}
\usepackage{amssymb}
\usepackage{amsthm}
\usepackage{algorithm}
\usepackage{algpseudocode}

\usepackage{booktabs}
\usepackage{todonotes}
\usepackage{adjustbox}
\usepackage{tabularx}
\usepackage{booktabs}

\usepackage{sectsty}

\usepackage{multirow}   
\usepackage{xcolor} 

\usepackage[most]{tcolorbox}

\usepackage{subcaption} 

\usepackage{wrapfig}

\usepackage{enumitem}

\tcbset{
    promptbox/.style={
        enhanced,
        colback=gray!5,
        colframe=black!70,
        boxrule=0.5pt,
        arc=2pt,
        left=6pt,
        right=6pt,
        top=6pt,
        bottom=6pt,
        breakable
    }
}

\newtheorem{definition}{Definition}
\newtheorem{assumption}{Assumption}
\newtheorem{lemma}{Lemma}
\newtheorem{theorem}{Theorem}
\newtheorem{proposition}{Proposition}
\newtheorem{corollary}{Corollary}
\newtheorem*{remark}{Remark}
\newcommand{\methodname}{FOCUS}

\title{\methodname{}: Training-Free Decision-Preserving Context Compression for LLM Agents}

\author{
Shantanu Dixit\quad
Anson Bastos\quad
Xuchao Zhang\quad
Chetan Bansal\quad
Saravan Rajmohan\\[4pt]
\multicolumn{1}{c}{\normalfont M365 Research, Microsoft}
}
\iclrfinalcopy 
\begin{document}

\maketitle
\lhead{Preprint. Under review.}

\begin{abstract}
LLM agents accumulate interaction histories that grow linearly with task length, causing quadratic inference cost scaling and performance degradation from attention dilution. 
Existing context-compression methods learn what to discard offline: by contrastively optimizing guidelines, distilling compressors, or training compression policies. This incurs a substantial cost. Further, the compression policy is learned a priori and is not dynamically conditioned on the evolving test-time trajectories. In this paper we ask a complementary question: \textit{Which past interactions causally shape the agent's future decisions?} We recast context compression as a causal decision preservation problem over discrete interaction units and introduce \textbf{FOCUS}, a training-free context compression framework that operates entirely at test time. Our method requires no offline data collection or fine-tuning, and is architecture-agnostic, attaching to any closed-API frontier model as a modular compression layer. We evaluate FOCUS on diverse agentic benchmarks including API and tool-calling, QA, web domain and multi-turn dialogue. Our method establishes new state of the art performance, cutting peak context by up to 48\% and dependency by 73\% while improving task success by up to 8.9 percentage points over uncompressed execution.
\end{abstract}

\section{Introduction}
Recent advances in large language models (LLMs) have enabled long-horizon agents that solve
complex tasks through iterative reasoning and interaction with external environments
\citep{yao2022react, trivedi2024appworld, wang2024officebench}.
As an agent acts, its context accumulates an ever-growing history of reasoning, actions, and
observations. We refer to a single (reasoning, action, observation) tuple as a \emph{span}.
Because this history grows linearly with the task horizon, and self-attention is quadratic in
sequence length, inference cost scales quadratically. Moreover, decision quality degrades as relevant evidence is diluted among accumulated context \citep{liu2024lost}. Long-horizon agents therefore face a fundamental bottleneck of retaining the historical context needed for good decisions under a strict context budget.

Existing approaches to context compression largely treat the problem as one of \emph{redundancy
reduction}. Token-level methods \citep{jiang2023llmlingua,
jiang2024longllmlingua} drop tokens according to a language-model likelihood, and soft-prompt
methods compress the context into summary vectors \citep{chevalier2023adapting}. More recent
agent-specific methods \citep{kang2025acon,yuksel2025paace} learn
compression strategies offline. A complementary axis reduces cost at the attention level through
KV-cache eviction \citep{zhang2023h2o, li2024snapkv}. However, these approaches face two
limitations: First, token and cache level filtering can corrupt
structured interaction traces (tool calls, code, or execution logs), by discarding syntactically
predictable but semantically critical elements. Second, methods that rely on offline supervision
incur substantial optimization or distillation cost. As their compression policy is
fixed before inference, 
they provide no mechanism to account for task-specific dependencies that emerge at inference time.
Reinforcement-learning approaches \citep{sun2025scaling} improve adaptivity but add optimization complexity and require
open-weight access for training.

In contrast, we argue that context compression for agents is fundamentally not a redundancy reduction problem, but a \emph{causal decision-preservation problem}. The central question we ask is: \textbf{\emph{which past spans are causally necessary to preserve the agent’s future decision trajectory?}} The relevance of a past interaction is therefore inherently forward-looking: a span should be retained if and only if it influences the agent’s future decision trajectory. This perspective suggests that the goal of compression is not to faithfully summarize the past, but to preserve a sufficient statistic of the history required for future decision-making.

From an information-theoretic perspective, the objective, of compressing the history while preserving information relevant to future decisions, corresponds to an Information Bottleneck (IB) objective. Directly optimizing the Information Bottleneck objective is intractable because it requires both (i) searching over an exponentially large space of possible span subsets, and (ii) evaluating mutual information terms involving the distribution of future trajectories in closed form.
Thus, instead of directly optimizing for selecting span subsets, we focus on the utility of each span for the task. For this, we seek to estimate the extent to which removing the span changes the agent's predicted future plans (\emph{counterfactual utility}). We formally show connections between the IB objective and the \emph{counterfactual utility} of each span. 
While this formulation is computationally tractable (though still expensive), it has the following caveat: the probability distribution over actions could be unavailable in an online setting.
We avoid the need for computing the exact probability distributions over the action space and compute span-wise utility. Specifically we design an approximation based on \emph{Monte Carlo Rollouts} of the future trajectory and propose \textbf{\methodname{}: Forward-looking Causal Utility Span Estimation}, a training-free framework for context compression that operates entirely at test time.

\methodname{} offers several advantages. First, it eliminates the need for offline data collection, synthetic trajectory generation, or task-specific fine-tuning, operating entirely through test-time reasoning. Second, it is architecture-agnostic and can be attached as a modular component to both open-weight and closed-API frontier models. Third, by operating over interaction spans rather than tokens, \methodname{} preserves the structural integrity of agent trajectories, ensuring that local causal relationships are not disrupted. Fourth, it also eliminates the need to compute exact probability distributions over unknown future trajectories while keeping the computation tractable.

Through extensive experiments on benchmarks: OfficeBench, AppWorld, 8-QA, WebVoyager and $\tau^2$-Bench, we demonstrate that \methodname{} significantly reduces context footprint while improving task success rates relative to existing compression methods. Further we empirically 
show that the additional planning overhead is offset by reduced input-token consumption, yielding lower total token (up to 31\%) usage while improving task success (Table~\ref{tab:latency_appendix}).
Our results highlight the importance of forward-looking, decision-aware compression and establish \methodname{} as a simple, effective and practical solution for scalable agent deployment.

\vspace{0.5em}
\noindent\textbf{Contributions.} We summarize our contributions as follows:
\begin{itemize}[leftmargin=*]
    
    \item We formulate context compression for agents as a causal decision-preservation problem, and connect it to an Information Bottleneck objective that minimizes retained history while preserving predictive information about future actions.
    
    \item We propose \methodname{}, a training-free, test-time compression framework that approximates the information bottleneck optimization using Monte Carlo rollouts of future trajectories. 
    
    \item We introduce a dual-objective strategy combining stochastic plan-based dependency estimation with a verification step that preserves critical, overlooked spans for safe and efficient compression.

    \item We empirically demonstrate that \methodname{} achieves state-of-the-art performance on agent benchmarks, improving task success rates while significantly reducing context usage.
\end{itemize}
\section{Related Work}

\paragraph{Long-horizon LLM agents.} LLMs are increasingly deployed as agents that perform iterative decision-making through repeated interaction with tools and environments \citep{yao2022react, qin2024toolllm}, coordinating actions across tens to
hundreds of steps while maintaining consistency with past observations and reasoning \citep{wang2024officebench, trivedi2024appworld}. As trajectories grow, the accumulated history becomes a major bottleneck, motivating context-management mechanisms that retain only information needed for future decisions.

\paragraph{Context compression.} A large body of work reduces effective context length. Token-level methods prune inputs by language-model likelihood or learned token importance \citep{jiang2023llmlingua, jiang2024longllmlingua, pan2024llmlingua}, and soft-prompt methods compress context into summary vectors \citep{chevalier2023adapting}; orthogonally, KV-cache eviction reduces cost at the attention level \citep{zhang2023h2o, li2024snapkv}. These
techniques suit static or loosely structured inputs, but are ill-suited to agent trajectories, whose structured tool calls, code, and execution logs carry causal dependencies that token or cache-level filtering can silently corrupt.

\paragraph{Context compression for LLM agents.} Agent-specific methods instead learn \emph{what to compress} offline. Early approaches use environment-specific strategies tailored to particular settings \citep{deng2023mind2web, yang2024swe, lee2025learning}, limiting
generality. Recent methods learn general policies: ACON \citep{kang2025acon} contrastively optimizes a natural-language compression guideline from full-versus-compressed trajectories and
distills it into smaller models, while PAACE \citep{yuksel2025paace} distills a plan-aware compressor from a synthetic workflow corpus (Table~\ref{tab:focus_comparison}). A second line learns context management jointly with the agent via reinforcement learning \citep{zhou2025mem1, sun2025scaling}. These incur substantial offline cost and, with a policy fixed before inference, cannot adapt to task-specific dependencies at test time; RL variants further require open-weight access. A complementary direction adds external memory \citep{packer2023memgpt, chhikara2025mem0},
reintroducing indexing and retrieval and targeting conversational recall rather than tool-using agentic control. 

Complementary to these, we cast compression as a \emph{decision-preserving causal abstraction} problem: \textbf{FOCUS} is a \emph{training-free, test-time} framework that scores each span by its \emph{forward-looking counterfactual utility} rather than by past redundancy or an offline policy. Operating at the span level, it is
orthogonal to token, cache, and memory-based methods, which can be applied within retained spans. 
\section{Problem Formulation}
\label{sec:problem_formulation}
In this section we formalize our approach for context compression in language model agents. Given a task goal \(g\), an agent interacts with an external environment over multiple steps by producing actions and receiving observations. The interaction history up to step \(t\) is denoted as $H_t = (s_1, s_2, \ldots, s_t),$
where each \(s_i\) is an atomic interaction span. In this work, we define a span as a complete reasoning-execution-feedback unit: $s_i = (r_i, a_i, o_i),$
where \(r_i\) denotes the agent's intermediate reasoning or thought, \(a_i\) denotes the executed action or tool call, and \(o_i\) denotes the resulting environment observation. We use span-level units rather than token-level units because agent trajectories often contain structured actions, tool calls, error messages, and environment states whose local causal structure should be preserved during compression.

An agent typically conditions its next action on the full interaction history and the task goal: $a_{t+1} \sim \pi(\cdot \mid H_t, g)$
where \(\pi\) denotes the main agent policy. As the horizon grows, however, \(H_t\) can become increasingly long, leading to higher inference cost, increased latency, and degraded decision quality. The goal of context compression is therefore to construct a compressed trace $Z_t = C(H_t),$
where \(C\) is a compression function and \(Z_t\) is the subset of spans substantially shorter than \(H_t\), while preserving the information required for future decision-making.

Unlike generic text summarization, agent context compression is not required to faithfully reconstruct every past event. Instead, the compressed trace should preserve the causal states and constraints that affect the agent's future behaviour. Let \(\tau_{t:T}\) denote the future trajectory from step \(t\) to the terminal step \(T\): $\tau_{t:T} = (a_{t+1}, o_{t+1}, \ldots, a_T, o_T).$
An ideal compressed trace should act as an approximate sufficient statistic of the full history for predicting future behaviour under the task goal: $P(\tau_{t:T} \mid Z_t, g) \approx P(\tau_{t:T} \mid H_t, g)$
while satisfying a smaller context budget, i.e., $|Z_t| \ll |H_t|.$

We therefore formulate context compression as a
decision-preserving compression problem. Let
$
P_H := P(\tau_{t:T}\mid H_t,g)
$
denote the future trajectory distribution conditioned on the full
interaction history $H_t$ and goal $g$, and let
$
P_Z := P(\tau_{t:T}\mid Z_t,g)
$
denote the Bayes optimal distribution conditioned on the compressed
context $Z_t = C(H_t)$.
The compressor seeks to minimize context cost while preserving the
decision-relevant information contained in the original history:
\begin{equation}
\begin{aligned}
\min_C \quad
& \mathrm{Cost}(Z_t)
+ \lambda \,
\mathbb{E}_{H_t,g}
\!\left[
D_{\mathrm{KL}}(P_H \,\|\, P_Z)
\right]
\\
\end{aligned}
\label{eq:decision_preserving}
\end{equation}
The first term measures context cost, while the second term penalises loss of decision-relevant information by comparing the future trajectory distribution induced by the compressed trace against that induced by the full history.

This objective admits a natural Information Bottleneck (IB) interpretation. Specifically, the compression term \(\mathrm{Cost}(Z_t)\) can be viewed as a proxy for the information retained from the original history, i.e.\ \(I(H_t; Z_t \mid g)\). As we keep adding tokens in $Z_t$ the mutual information between $H_t, Z_t$ increases and so does $\mathrm{Cost}(Z_t)$.
The expected KL term is equivalent, up to an additive constant independent of \(C\), to the conditional entropy of the future trajectory given the compressed trace. Let $\mathcal{H}_Z := \mathcal{H}_e(\tau_{t:T}\mid Z_t,g)$ and $\mathcal{H}_H := \mathcal{H}_e(\tau_{t:T}\mid H_t,g)$, where $\mathcal{H}_e$ is the entropy, then:
$\mathbb{E}_{H_t,g}\!\bigl[D_{\mathrm{KL}}(P_H \,\|\, P_Z)\bigr] = \mathcal{H}_Z - \mathcal{H}_H.$
Since \(\mathcal{H}_e(\tau_{t:T}\mid H_t,g)\) does not depend on the compressor \(C\), minimizing the expected trajectory divergence is equivalent to minimizing \(\mathcal{H}_e(\tau_{t:T}\mid Z_t,g)\), or equivalently, maximising the predictive information
$
I(Z_t;\tau_{t:T}\mid g)
=
\mathcal{H}_e(\tau_{t:T}\mid g)-\mathcal{H}_e(\tau_{t:T}\mid Z_t,g).
$
Thus, our objective can be recast in the IB form
\[
\min_C\; I(H_t; Z_t \mid g)\;-\;\beta\, I(Z_t;\tau_{t:T}\mid g),
\]
for some trade-off coefficient \(\beta > 0\). In this view, context compression seeks a representation \(Z_t\) that is maximally compact while preserving the information in the history that is most relevant for predicting the agent's future trajectory.
The central challenges are: (i) obtaining mutual information requires future trajectory distributions which are not directly observable (ii) optimizing the objective comparing distributions is computationally expensive (c.f. Fig. \ref{fig:draft_convergence}). In the next section we propose our method that enables a principled approximation to optimizing the above objective.

\section{Method: FOCUS}

We propose a test-time context compression method, named FOCUS, for language model agents. FOCUS compresses an agent's interaction history according to its estimated utility for future decision-making, rather than according to token-level redundancy or generic summarization quality. When the accumulated history exceeds a context budget, FOCUS identifies which historical spans are likely to affect the agent's future behaviour, preserves those spans
and removes spans that are unlikely to influence downstream decisions.

At a high level, FOCUS follows a four-stage pipeline. First, it represents the agent history as a sequence of atomic interaction spans. Second, it estimates the future utility of each span through draft-model plan rollouts. Third, it retains spans that are repeatedly cited as future dependencies. Finally, it performs defensive verification to rescue spans whose deletion may cause repeated mistakes or loss of causal state. Algorithm~\ref{alg:psa} and Figure \ref{fig:psa_method_overview} summarizes the overall procedure.

\begin{algorithm}[t]
\caption{FOCUS}
\label{alg:psa}
\begin{algorithmic}[1]
\Require History $H_t=(s_1,\ldots,s_t)$, goal $g$, budget $\delta_{\text{mem}}$, draft model $q_\phi$, rollouts $N$, threshold $\tau$
\Ensure Compressed trace $Z_t$

\If{$|H_t| \leq \delta_{\text{mem}}$}
    \State \Return $H_t$
\EndIf

\State Initialize $\hat{u}_i \leftarrow 0$ for all $s_i \in H_t$

\For{$k=1$ to $N$}
    \State Sample plan sketch $p_k \sim q_\phi(p\mid H_t,g)$
    \State Extract dependency set $\mathrm{Dep}(p_k)$
    \State Update $\hat{u}_i \leftarrow \hat{u}_i + 1/N$ for each $s_i\in \mathrm{Dep}(p_k)$
\EndFor

\State $S_U \leftarrow \{s_i\in H_t:\hat{u}_i\geq\tau\}$

\State $S_R \leftarrow \mathrm{DefensiveVerify}(H_t\setminus S_U, q_\phi, \tau)$

\State $S_{\mathrm{keep}} \leftarrow S_U \cup S_R$

\State $S_{\mathrm{drop}} \leftarrow H_t \setminus S_{keep}$

\State $Z_t \leftarrow \mathrm{BuildTrace}(S_{\mathrm{keep}},S_{\mathrm{drop}})$

\State \Return $Z_t$
\end{algorithmic}
\end{algorithm}

\subsection{Span-Level State Representation}

As defined in Section~\ref{sec:problem_formulation}, an interaction history $H_t$ is constructed from atomic interaction spans $s_i = (r_i, a_i, o_i)$. By operating strictly at this span level rather than the token level, FOCUS prevents the fragmentation of structured actions, tool calls, error messages, and observations. This span-level integrity ensures the local causal relation between the agent's intent, execution, and environmental feedback remains uncorrupted, allowing the state transitions to be cleanly evaluated for future utility.\\
The compressor is triggered when the history exceeds a predefined context budget $\delta_{\text{mem}}$, i.e., $|H_t|>\delta_{\text{mem}}$. It then constructs a compressed trace $Z_t=C(H_t)$, which is used by the main agent for subsequent decision-making: $a_{t+1} \sim \pi(\cdot \mid Z_t,g).$

\begin{figure*}
    \centering
    \includegraphics[width=0.95\linewidth]{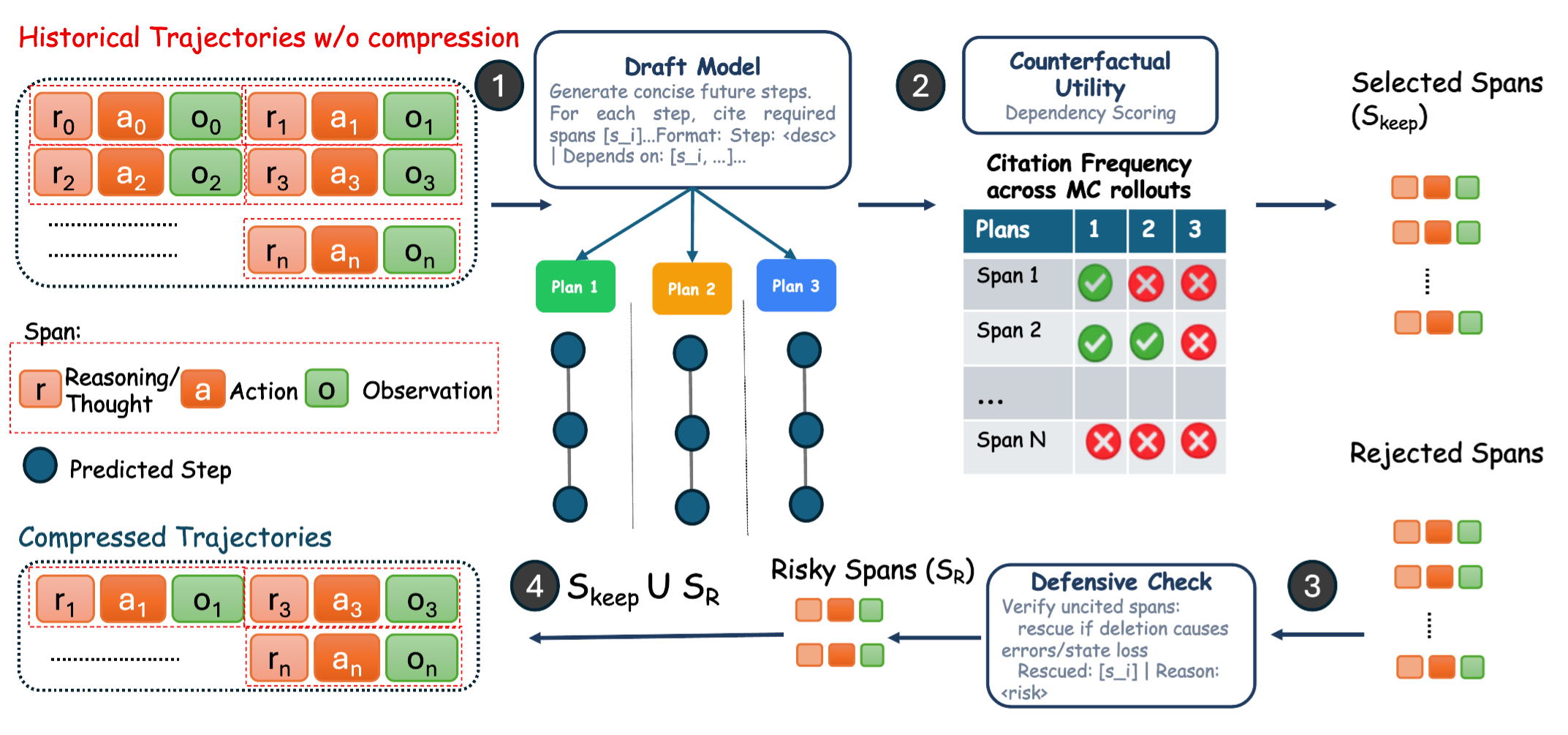}
    \caption{
Overview of FOCUS for decision-preserving context compression. Given a growing interaction history composed of thought–action–observation spans, FOCUS invokes a lightweight draft model to generate multiple stochastic plan sketches, each explicitly identifying the historical spans required for future steps. These dependencies are aggregated to estimate counterfactual future utility, enabling the selection of spans that significantly influence downstream decisions. To ensure robustness in stateful environments, a defensive verification stage identifies and retains uncited spans whose removal may lead to repeated errors or loss of critical state. 
}
    \label{fig:psa_method_overview}
\end{figure*}

\subsection{Counterfactual Future Utility}

The core question in FOCUS is whether a historical span is necessary for future decisions. 
By definition this intends to measure whether removing the span would change the future behaviour of the agent. In this section, we also show connections of this measure with the IB objective in \S \ref{sec:problem_formulation}.

For a span $s_i$, let $H_t^{-i}=H_t\setminus\{s_i\}$ denote the counterfactual history obtained by removing $s_i$.

\begin{definition}[Counterfactual Future Utility]
For a future variable $Y$, such as the remaining trajectory or a high-level future plan, the counterfactual future utility of span $s_i$ is defined as
\begin{equation}
U(s_i)
=
D\!\left(
P(Y\mid H_t,g)
\;\|\;
P(Y\mid H_t^{-i},g)
\right),
\label{eq:counterfactual-utility}
\end{equation}
where $D(\cdot\|\cdot)$ is a divergence measure (e.g. KL divergence) between future distributions.
\end{definition}

A span with high $U(s_i)$ is future-relevant because removing it changes the predicted future behaviour, whereas a span with low $U(s_i)$ can potentially be abstracted or removed.

$U(s_i)=0$ when removing $s_i$ leaves the future decision
distribution unchanged. 
contrapositively, $U(s_i) > 0$ implies removing $s_i$ changes the agent's future decisions. $U$ therefore formalizes the core intuition of FOCUS: a span is worth preserving insofar as it affects the agent's future decisions.
We next show that this notion
of counterfactual utility, under mild assumptions
(Assumption~\ref{assumption:future-coverage-invariance}), corresponds to a
principled information-theoretic objective: \emph{selecting spans by counterfactual
utility under a context budget optimizes the Information Bottleneck (IB)
objective up to a bounded interaction term.}

\begin{proposition}[Counterfactual utility and the Information Bottleneck objective]
Let $Z_t\subseteq H_t$ be a compressed trace satisfying the entropy budget
$I(H_t;Z_t\mid g)=\mathcal{H}(Z_t\mid g) - \mathcal{H}(Z_t\mid H_t,g) \leq \mathcal{H}(Z_t\mid g)\le B$, and let $Y$ denote the random set of
future events. Then $\mathbb{E}[U(s_i)]=I(s_i;Y\mid H_t^{-i},g)$, and under
Assumption~\ref{assumption:future-coverage-invariance} with tolerance
$\bar\varepsilon \geq 0$, we have $\Bigl|\,I(Z_t;Y\mid g)-\textstyle\sum_{s_i\in Z_t}\mathbb{E}[U(s_i)]\Bigr|\le\bar\varepsilon .$
Hence maximising total counterfactual utility subject to the budget maximises
$I(Z_t;Y\mid g)$ up to $\bar\varepsilon$, and its Lagrangian is the Information Bottleneck objective
\[
\min_C \; I(H_t;Z_t\mid g) - \beta\, I(Z_t;Y\mid g),
\]
with $\beta$ the inverse multiplier of the budget. 
The error
$\bar\varepsilon$ vanishes for long horizons (Appendix~\ref{ib_proof}).
\end{proposition}

However, exact computation of the KL-based utility requires evaluating future trajectory distributions under counterfactual histories, which are not directly observable at inference time.
FOCUS approximates this using an efficient count-based utility obtained from draft-model plan rollouts.

\subsection{Draft-Model Monte Carlo Estimation}

In this section we present our approximation to optimizing the counterfactual utility (or IB) objective seen in the previous sections.
Let $q_\phi$ denote a lightweight draft model used by FOCUS. Given the task goal $g$ and history $H_t$, FOCUS samples $N$ stochastic plan sketches, $p_k \sim q_\phi(p\mid H_t,g)$ for $k=1,\ldots,N$. Each plan sketch describes a possible high-level plan for completing the remaining task. For every planned step, the draft model is required to cite the historical spans it depends on and
outputs a short description together with a dependency set, e.g.,
\texttt{Step: ... | Depends on: [s\_i, s\_j]}.
We say the future depends on $s_i$ if some remaining step requires information contained in $s_i$ and the rollouts sample this event.
This converts future planning into an explicit dependency estimation problem. For each span $s_i$, we define a binary dependency indicator:
$
X_i^{(k)}
=
\mathbf{1}\!\left[s_i \in \mathrm{Dep}(p_k)\right],
$
, where $\mathrm{Dep}(p_k)$ denotes the set of spans cited by plan sketch $p_k$. FOCUS estimates the future dependency score of $s_i$ by its citation frequency across rollouts:
$
\hat{u}_i
=
\frac{1}{N}\sum_{k=1}^{N} X_i^{(k)}.
$

Intuitively, $\hat{u}_i$ measures how consistently the draft model identifies $s_i$ as necessary for future planning. Spans repeatedly cited across stochastic rollouts are more likely to contain information needed by the main agent.

In the below result we connect this estimation of span importance with the solution optimizing the counterfactual utility (information bottleneck). Proof is in Appendix \ref{monotone_ranking}.

\begin{proposition}[Dependency score bounds counterfactual utility]
\label{prop:main-certificate}
Under mild assumptions (Appendix~\ref{monotone_ranking}), for any set
of spans $S\subseteq H_t$,
\[
D_{\mathrm{KL}}\bigl(P(Y\mid H_t,g)\,\|\,P(Y\mid H_t\setminus S,g)\bigr)
\;\le\; m_{\max}\,\mathbb{P}\bigl( \exists s_i \in S: X_i = 1) \leq m_{\max}\sum_{s_i\in S} \hat{u_i}
\]

\end{proposition}

In the above, $m_{\max}$ is the largest distortion that removing a set of spans
induces on the futures that depend on it (Assumption~\ref{ass:impact-scaling}).
Thus sets of spans that plan rollouts rarely cite 
have provably small counterfactual utility, and discarding them
provably preserves the agent's future decision distribution up to
$m_{\max}$ times their citation probability.
In the appendix \ref{sec_utility_separation} we show that the MC estimate $\hat{u_i}$ is unbiased and informs the number of rollouts needed.

Given the estimated score $\hat{u}_i$, FOCUS first constructs an initial utility-retained set:
$
S_U
=
\{s_i \in H_t : \hat{u}_i \geq \tau\}
$,
where $\tau$ is a utility threshold. Spans in $S_U$ are preserved because they are estimated to have high future planning utility.
For brevity, the draft model prompts and rollout examples are provided in Appendix 
\ref{draft_model_prompt}.

\subsection{Defensive Verification}
\label{method: defensive_verification}
The previous section's utility-based planning may miss spans that encode negative constraints or stateful failures. For example, a failed tool call, invalid password attempt, or a rejected API request may not be explicitly cited in a future plan. However, removing such spans can cause the main agent to repeat invalid actions, enter invalid loops, or lose important state information.
To address this issue, FOCUS includes a defensive verification stage. After optimistic planning identifies the utility-retained set $S_U$, the draft model reviews spans not included in $S_U$ and identifies spans whose deletion may cause repeated mistakes, invalid loops, or loss of causal state. This produces a rescued set: $S_R$.
The final preserved set is: $S_{\mathrm{keep}}
=
S_U
\cup
S_R.$

Spans in $S_R$ are retained even if they were not frequently cited during optimistic planning. This defensive step prioritizes avoiding catastrophic forgetting over maximum compression.
The final compressed trace retains spans in $S_{\text{keep}}$ verbatim and discards the remainder: $Z_t = \{s_i \in H_t : s_i \in S_{\text{keep}}\}$
where $S_{\text{keep}} = S_U \cup S_R$ as defined above. Appendix \ref{sec_missed_spans} provides conditions when this stage helps.
We present further theoretical analysis and experiments validating the design choices in Appendix~\ref{theoretical_analysis}.

\section{Experiments}
\label{sec:experiments}

We address the following questions through our experiments:
1) \textbf{RQ1 (Efficiency vs. Performance):} Can \methodname{} reduce context costs (peak tokens and dependency) without degrading task success compared to uncompressed and existing compression baselines? (\S\ref{sec:main_results})
2) \textbf{RQ2 (Draft Model Sensitivity):} How sensitive is the compression efficacy and overall API cost to the size and capability of the draft model? (\S\ref{sec:draft_models})
3) \textbf{RQ3 (Memory Budget and Compression Frequency):} How does the context budget threshold $\delta_{\text{mem}}$ dictate the frequency of compressions, and what is its effect on the efficiency-accuracy trade-off? (\S\ref{sec:ablations})
4) \textbf{RQ4 (Cost and Latency):} What is the API cost and wall-clock latency overhead of the draft-based estimator? (\S\ref{sec:cost_latency})

\textbf{Experimental Setup:}
\label{sec:setup}
We evaluate \methodname{} on five agentic benchmarks spanning diverse agentic settings, including API-centric tool use, knowledge intensive question answering, web navigation and multi-turn dialogue: \textbf{AppWorld}~\cite{trivedi2024appworld}, \textbf{OfficeBench}~\cite{wang2024officebench}, \textbf{8-objective QA}~\cite{kwiatkowski2019natural}, \textbf{WebVoyager} \cite{he2024webvoyager} and \textbf{$\tau^2$-Bench} \cite{barres2025tau}. 
We characterize the trajectory statistics of the benchmarks in Table \ref{tab:long_horizon_stats}. 
We evaluate our method using \textbf{Accuracy} (task success), \textbf{Steps} (avg. interactions), \textbf{Peak Tokens} (max. context length across steps, $10^3$), \textbf{Dependency} (cumulative dependence of actions on prior context, $10^6$). 
\textbf{Implementation:} For fairness, we adopt the same experiment settings as ACON (\cite{kang2025acon}, ICML). We use \texttt{gpt-4.1} (main agent, temperature 0.0, seed 42), with tokenization using \texttt{tiktoken} (\texttt{cl100k\_base}) and optional open-weight draft models (Qwen3-8B/14B via HuggingFace). We detail the benchmarks and implementation details in ~\ref{app:benchmarks} and ~\ref{app:implementation_details} respectively.

\subsection{Results}
\label{sec:main_results}


\begin{table}[t]
\centering

\begin{minipage}[t]{0.32\linewidth}
\centering
\scriptsize
\setlength{\tabcolsep}{0.8pt}
\renewcommand{\arraystretch}{0.9}

\begin{adjustbox}{max width=\linewidth}
\begin{tabular}{lcccc}
\toprule
\textbf{Method} & \textbf{Acc$\uparrow$} & \textbf{Stp$\downarrow$} & \textbf{Pk$\downarrow$} & \textbf{Dep$\downarrow$} \\
\midrule
\multicolumn{5}{c}{\textbf{Agent:} \texttt{gpt-4.1} / \textbf{Comp:} \texttt{gpt-4.1}} \\
\midrule
No compression & 56.0 & 16.14 & 9.93 & 5.96 \\
\midrule
FIFO & 45.8 & 28.48 & \underline{6.73} & 5.69 \\
Retrieval & 27.4 & 33.17 & 8.39 & 6.68 \\
LLMLingua & 39.3 & 24.42 & 7.50 & 6.37 \\
Prompting & 43.5 & 24.01 & 6.93 & 5.29 \\
ACON UT & 51.2 & 20.92 & 7.17 & 4.49 \\
ACON UTCO & \underline{56.5} & 22.82 & 7.33 & 4.69 \\
\rowcolor{blue!10}
FOCUS-O & \underline{56.5} & \underline{18.90} & \textbf{6.50} & \textbf{2.38} \\
\rowcolor{blue!10}
FOCUS-D & \textbf{64.9} & \textbf{16.10} & 8.37 & \underline{4.15} \\
\bottomrule
\end{tabular}
\end{adjustbox}

\subcaption{AppWorld \texttt{test\_normal}.}
\label{tab:appworld}
\end{minipage}
\hfill
\begin{minipage}[t]{0.32\linewidth}
\centering
\scriptsize
\setlength{\tabcolsep}{0.8pt}
\renewcommand{\arraystretch}{0.9}

\begin{adjustbox}{max width=\linewidth}
\begin{tabular}{lcccc}
\toprule
\textbf{Method} & \textbf{Acc$\uparrow$} & \textbf{Stp$\downarrow$} & \textbf{Pk$\downarrow$} & \textbf{Dep$\downarrow$} \\
\midrule
\multicolumn{5}{c}{\textbf{Agent:} \texttt{gpt-4.1} / \textbf{Comp:} \texttt{gpt-4.1}} \\
\midrule
No compression & 76.84 & 11.52 & 7.27 & 4.43 \\
\midrule
FIFO & 67.37 & 12.26 & \underline{4.02} & 2.64 \\
Retrieval & 65.26 & 16.20 & 4.33 & 2.06 \\
LLMLingua & 70.53 & 10.89 & 4.65 & 1.85 \\
Prompting & 71.58 & 10.13 & 4.40 & \textbf{1.10} \\
ACON UT & 74.74 & 13.13 & 4.93 & 3.85 \\
ACON UTCO & 72.63 & 11.54 & 4.54 & 1.91 \\
\rowcolor{blue!10}
FOCUS-O & \textbf{78.90} & \textbf{9.30} & \textbf{3.81} & \underline{1.16} \\
\rowcolor{blue!10}
FOCUS-D & \underline{77.90} & \underline{9.60} & 4.20 & 1.36 \\
\bottomrule
\end{tabular}
\end{adjustbox}

\subcaption{OfficeBench \texttt{test}.}
\label{tab:officebench}
\end{minipage}
\hfill
\begin{minipage}[t]{0.32\linewidth}
\centering
\scriptsize
\setlength{\tabcolsep}{0.7pt}
\renewcommand{\arraystretch}{0.9}

\begin{adjustbox}{max width=\linewidth}
\begin{tabular}{lccccc}
\toprule
\textbf{Method} & \textbf{EM$\uparrow$} & \textbf{F1$\uparrow$} & \textbf{Stp$\downarrow$} & \textbf{Pk$\downarrow$} & \textbf{Dep$\downarrow$} \\
\midrule
\multicolumn{6}{c}{\textbf{Agent:} \texttt{gpt-4.1} / \textbf{Comp:} \texttt{gpt-4.1}} \\
\midrule
No compression & 0.366 & 0.488 & 15.78 & 10.35 & 3.32 \\
\midrule
FIFO & 0.293 & 0.388 & 19.26 & 5.09 & 2.51 \\
Retrieval & 0.331 & 0.438 & 20.06 & 5.11 & 2.62 \\
LLMLingua & 0.363 & 0.481 & 17.68 & 5.68 & 2.24 \\
Prompting & \underline{0.376} & 0.478 & 18.70 & 4.73 & 1.66 \\
ACON UT & 0.373 & 0.494 & 17.14 & \underline{4.71} & \underline{1.57} \\
ACON UTCO & 0.335 & 0.458 & 17.79 & \textbf{4.65} & \textbf{1.50} \\
\rowcolor{blue!10}
FOCUS-O & \textbf{0.386} & \textbf{0.500} & \textbf{16.10} & 6.70 & 2.30 \\
\rowcolor{blue!10}
FOCUS-D & \underline{0.376} & \underline{0.498} & \underline{17.00} & 9.17 & 4.05 \\
\bottomrule
\end{tabular}
\end{adjustbox}

\subcaption{8-objective QA \texttt{test}.}
\label{tab:8objqa}
\end{minipage}

\caption{Results on AppWorld, OfficeBench, and the 8-objective QA benchmark. \methodname{} improves task performance while reducing peak tokens and dependency.}
\label{tab:all_benchmarks}
\end{table}

\begin{table*}[t]
\centering
\footnotesize

\begin{minipage}[t]{0.485\textwidth}
\vspace{0pt}
\centering
\setlength{\tabcolsep}{3pt}
\renewcommand{\arraystretch}{0.9}

\begin{adjustbox}{max width=\linewidth}
\begin{tabular}{lcccc}
\toprule
\textbf{Method}
& \textbf{Acc. $\uparrow$}
& \textbf{Steps $\downarrow$}
& \textbf{Peak $\downarrow$}
& \textbf{Dep. $\downarrow$} \\
\midrule

\multicolumn{5}{c}{
\textbf{Agent:} \texttt{gpt-4.1-mini} /
\textbf{Compressor:} \texttt{gpt-4.1-mini}
} \\
\midrule

No compression & 35.7 & 18.14 & 8.55 & 5.07 \\
FIFO           & 39.3 & 30.39 & \textbf{6.18} & 5.24 \\
Retrieval      & 14.9 & 40.18 & 7.49 & 5.95 \\
LLMLingua      & 36.3 & 28.41 & 7.24 & 6.65 \\
Prompting      & 35.7 & 24.98 & 6.56 & 4.95 \\
ACON UT        & 42.3 & 22.46 & 6.51 & 5.48 \\
ACON UTCO      & 32.7 & 24.27 & 6.99 & 4.97 \\

\rowcolor{blue!10}
FOCUS-D & \textbf{44.0} & \textbf{17.90} & 6.71 & \textbf{3.81} \\

\bottomrule
\end{tabular}
\end{adjustbox}

\caption{
Results on AppWorld using a \texttt{gpt-4.1-mini} agent and compressor.
\methodname{} improves task success while reducing context usage dependency.
}
\label{tab:appworld_gpt41mini_avg}
\end{minipage}
\hfill
\begin{minipage}[t]{0.485\textwidth}
\vspace{0pt}
\centering
\setlength{\tabcolsep}{3pt}
\renewcommand{\arraystretch}{0.9}

\begin{adjustbox}{max width=\linewidth}
\begin{tabular}{lcccc}
\toprule

\textbf{Method}
& \textbf{Acc. $\uparrow$}
& \textbf{Steps $\downarrow$}
& \textbf{Peak $\downarrow$}
& \textbf{Dep. $\downarrow$} \\

\midrule

\multicolumn{5}{c}{
\textbf{Agent:} \texttt{gpt-4.1}
} \\

\midrule

Prompting (\texttt{gpt-4.1-mini})
& 39.3 & 23.61 & 7.03 & 5.19 \\

ACON (\texttt{gpt-4.1-mini})
& 47.6 & 21.46 & 7.25 & 5.24 \\

ACON (\texttt{Qwen3-14B})
& 50.0 & 21.72 & 6.83 & 4.80 \\

ACON (\texttt{Qwen3-8B})
& 47.0 & 21.58 & 6.98 & 4.76 \\

ACON (\texttt{Phi-4})
& 44.6 & 21.19 & 7.24 & 4.76 \\

\rowcolor{blue!10}
\methodname{}-D (\texttt{gpt-4.1-mini})
& 58.3 & \textbf{18.0} & \textbf{6.70} & \textbf{2.45} \\

\rowcolor{blue!10}
\methodname{}-D (\texttt{Qwen3-14B})
& 61.3 & 18.2 & 7.26 & 3.00 \\

\rowcolor{blue!10}
\methodname{}-D (\texttt{Qwen3-8B})
& \textbf{65.5} & 18.4 & 7.44 & 3.00 \\

\rowcolor{blue!10}
\methodname{}-D (\texttt{Phi-4})
& 61.3 & 18.6 & \textbf{6.70} & 2.62 \\

\bottomrule
\end{tabular}
\end{adjustbox}

\caption{
Performance on AppWorld using \texttt{gpt-4.1} as the main agent paired
with various open-weight models as compressors. \methodname{}
increases average task success (up to 65.5\%) while reducing context
dependency.
}
\label{tab:appworld_open_models_avg}
\end{minipage}

\end{table*}
\textbf{Overall Performance (RQ1)}
Table ~\ref{tab:all_benchmarks} presents our performance using \texttt{gpt-4.1} for both the main and draft models. We evaluate two strategies: \textbf{FOCUS-O(ptimistic)} (retains spans explicitly cited for future steps) and \textbf{FOCUS-D(efensive)} corresponding to Section~\ref{method: defensive_verification}.
On \textbf{AppWorld}, \methodname{} yields strong simultaneous improvements. FOCUS-O reduces peak tokens and dependency by 35\% and 60\% respectively while retaining baseline accuracy. The defensive guardrail further improves the task success, jumping to 64.9\% absolute accuracy (+15.9\% relative to no compression) while still keeping the context costs low.
These trends hold across benchmarks. On \textbf{OfficeBench}, FOCUS surpasses all baselines with an accuracy of 78.9\% while reducing peak tokens and dependency by 47\% and 73\% correspondingly compared to the uncompressed baseline. On \textbf{8-QA}, \methodname{} improves the exact match (EM) scores while maintaining a tighter context budget. Indicating that utility-driven pruning acts as a powerful regularizer, decluttering the context window and guiding the agent toward accurate reasoning. 
On \textbf{WebVoyager}, \methodname{} improves accuracy by $+4.5$ over NC, reducing peak tokens by $27\%$ (Table~\ref{tab:webvoyager}). On \textbf{$\tau^2-$Bench} (Table~\ref{tab:tau2}), \methodname{} improves task success by $+5.5$ over NC while maintaining the context budget.
To further test the efficacy of \methodname{} under a weaker backbone, we additionally evaluate with the smaller \texttt{gpt-4.1-mini} as both the
main agent and the draft model (Table~\ref{tab:appworld_gpt41mini_avg}). Even in
this constrained setting, \methodname{} improves task success by $23\%$
relative to the no-compression baseline  while reducing cumulative dependency and peak tokens by $\sim25\%$ and $\sim21\%$ respectively. 

Overall, these results show that \methodname{} is not tied to a strong backbone: even a small agent compressing its own history with a same-size draft realizes substantial accuracy and efficiency gains, making decision-preserving compression a practical drop-in for cost-constrained deployments.

\begin{figure}[t]
    \centering
    \includegraphics[width=0.6\columnwidth]{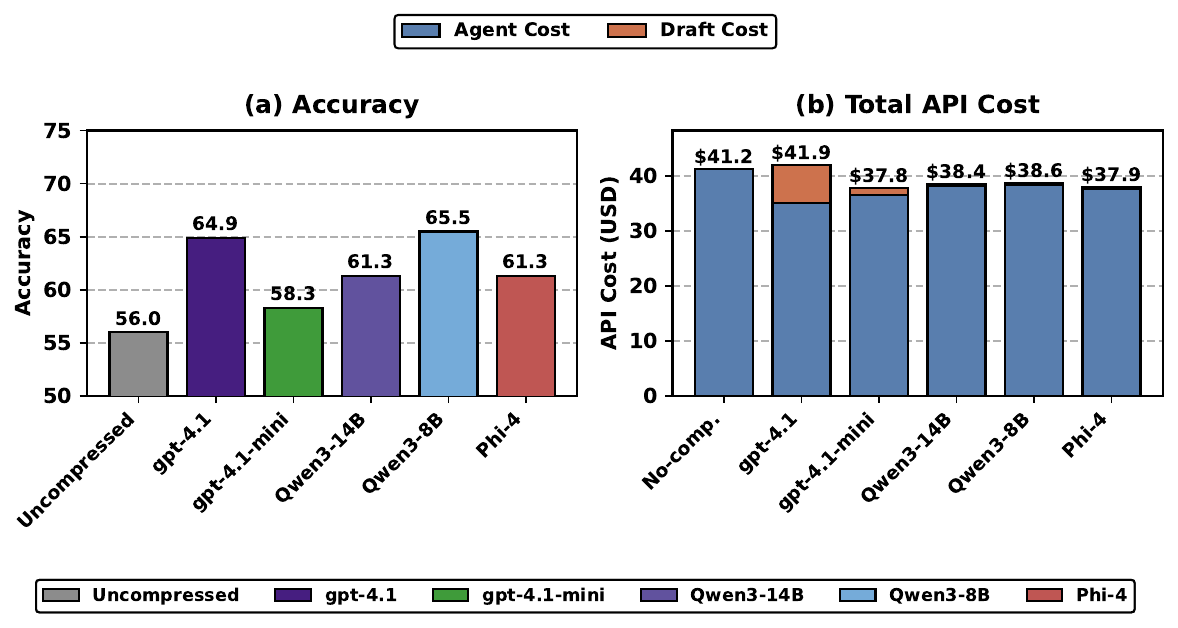}
    \caption{Impact of draft-model choice on \textbf{AppWorld} task accuracy and API cost, with \texttt{gpt-4.1} as the main agent. \textbf{(a)} Task accuracy across different open-source draft models; \methodname{} remains competitive even with small open-source drafts without any fine-tuning. \textbf{(b)} Total API cost (agent~+~draft, stacked); compressing the history reduces serving cost relative to the no-compression baseline.}
    \label{fig:appworld_draft_comparison}
\end{figure}

\textbf{Draft Model Independence (RQ2)}
\label{sec:draft_models}
To assess whether \methodname{} depends on a specific draft model, we fix the primary agent as \texttt{gpt-4.1} and experiment with different draft models across various model families and scales, including \texttt{gpt-4.1-mini}, \texttt{Qwen3-14B}, \texttt{Qwen3-8B}, and \texttt{Phi-4} (Table~\ref{tab:appworld_open_models_avg}, Figure~\ref{fig:appworld_draft_comparison}).
We observe that compression efficacy remains robust across all of them: every draft model preserves or improves task accuracy over the uncompressed baseline at comparable or reduced API cost, indicating that the utility signal driving span selection is not tied to a particular draft family.
For instance, \texttt{Qwen3-8B} as draft significantly improves main agent task performance by 17\% while  cutting average peak tokens by 25\%. With \texttt{gpt-4.1} as both agent and draft model we observe essentially equal cost (0.245 USD and 0.249 USD per task respectively), while smaller, cheaper drafts yield lower API cost and improvements over the uncompressed agent. We attribute this cost reduction to fewer tokens handled by the main agent, indicating that forward-looking span estimation clears unrequired context and helps the agent towards task completion. 
We provide detailed token statistics for these runs in Appendix Table ~\ref{tab:appworld_token_stats}.

\begin{table*}[t]
\centering
\footnotesize

\begin{minipage}[t]{0.48\textwidth}
\vspace{0pt}
\centering
\setlength{\tabcolsep}{2pt}
\renewcommand{\arraystretch}{0.9}

\begin{adjustbox}{max width=\linewidth}
\begin{tabular}{lccc}
\toprule
\textbf{Config}
& \textbf{Acc. $\uparrow$}
& \textbf{Steps $\downarrow$}
& \textbf{Peak (K) $\downarrow$} \\
\midrule
$N=1$  & 71.6 & \textbf{9.2} & 3.78 \\
\rowcolor{blue!10}
$N=3$  & \textbf{77.9} & 9.6 & 4.20 \\
$N=5$  & 74.7 & 9.6 & \textbf{3.74} \\
$N=10$ & 74.7 & 10.2 & 4.11 \\
\bottomrule
\end{tabular}
\end{adjustbox}

\caption{%
Ablation over the number of sampled trajectories ($N$) on OfficeBench,
using \texttt{gpt-4.1} as both the agent and draft model.
We observe diminishing returns beyond $N=3$ rollouts.
}
\label{tab:ablation_n}
\end{minipage}
\hfill
\begin{minipage}[t]{0.48\textwidth}
\vspace{0pt}
\centering
\setlength{\tabcolsep}{2pt}
\renewcommand{\arraystretch}{0.95}

\begin{adjustbox}{max width=\linewidth}
\begin{tabular}{lcc}
\toprule
\textbf{Method}
& \textbf{Latency (s) $\downarrow$}
& \textbf{API Cost (\$) $\downarrow$} \\
\midrule
No Compression
& 47.2
& 9.47 \\

\methodname{} (sequential)
& 40.2
& \textbf{5.96} \\

\rowcolor{blue!10}
\methodname{} (parallel)
& \textbf{36.5}
& 6.16 \\
\bottomrule
\end{tabular}
\end{adjustbox}

\caption{%
Mean wall-clock latency per task and total API cost
(agent~+~draft) on the OfficeBench test split, using
\texttt{gpt-4.1} as the agent and \texttt{gpt-4.1-mini} as the
 draft model. A complete token, request, and compression breakdown
is provided in Table~\ref{tab:latency_appendix}.
}
\label{tab:latency_cost}
\end{minipage}

\end{table*}

\noindent \textbf{Rollout Ablations: }
We perform a sensitivity study to show the effect of the number of Monte Carlo plan rollouts used to estimate span utility. As shown in Table \ref{tab:ablation_n}, increasing N from 1 to 3 substantially improves performance, raising task success from 71.6\% to 77.9\% on Officebench. 
This suggests that 
stochastic plan sketches provide a more stable estimate of future dependencies and reduces variance.
However, the gains quickly saturate beyond N=3.

\noindent \textbf{Memory Budget and Compression Frequency (RQ3): }
\label{sec:ablations}
We also evaluate the impact of the memory threshold $\delta_{\text{mem}}$ on compression frequency and the efficiency-performance trade-off on AppWorld. 
\begin{wrapfigure}{l}{0.45\linewidth} 
\vspace{-0.5em} 
\centering 
\includegraphics[width=\linewidth]{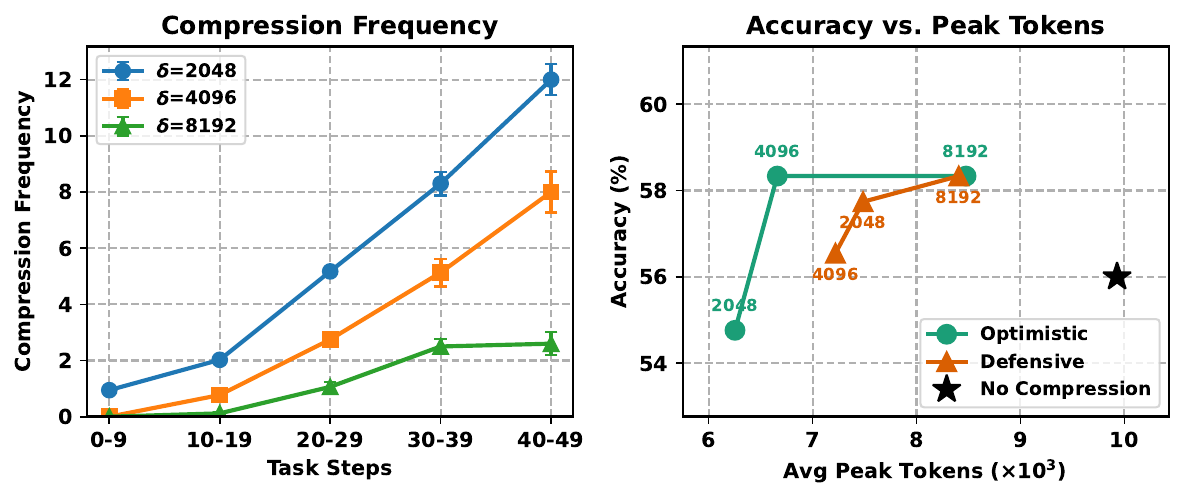} 
\caption{Memory budget ablations on AppWorld with \texttt{gpt-4.1-mini} as draft model. \textbf{Left:} Compression frequency increases with task length and decreases with larger $\delta_{\text{mem}}$. \textbf{Right:} Pareto frontier of accuracy vs.\ average peak tokens across $\delta_{\text{mem}}$ settings.} 
\label{fig:ablation} 
\vspace{-1em} 
\end{wrapfigure}
Figure~\ref{fig:ablation} (left) illustrates how $\delta$ controls the compression frequency. We observe that smaller thresholds trigger early and frequent compressions as tasks progress, whereas a relaxed threshold defers intervention until later steps. 
Figure~\ref{fig:ablation} (right) plots accuracy versus average peak token usage. We observe that lowering $\delta_{\text{mem}}$ reduces peak token consumption, with moderate thresholds providing the best trade-off between efficiency and accuracy.

\noindent \textbf{Cost and Latency Analysis (RQ4):}
\label{sec:cost_latency}
\methodname{} reduces total API cost by 37\% relative to the no-compression baseline 
 on OfficeBench (Table ~\ref{tab:latency_cost}). 
The resulting agent-side savings far outweigh the draft-model overhead, which accounts for only 3\% of the total cost. The $N$ rollouts are conditionally independent given a shared prompt and can therefore be issued concurrently rather than sequentially, lowering the per-event compression latency by $\sim\!2.5\times$ (refer Table~\ref{tab:latency_appendix}) and reducing mean wall-clock latency by $\sim\!10\%$ at essentially identical token cost. 
A full per-stage breakdown is reported in Table~\ref{tab:latency_appendix}.

\section{Conclusion}

In this work, we introduce FOCUS, a training-free framework for context compression in language model agents. We reformulate context compression as a decision-preserving causal abstraction problem, rather than a redundancy reduction task, and show that this formulation aligns with an Information Bottleneck objective that balances compression with preservation of future decision-relevant information. We further derive a novel efficient approximation of the IB objective through stochastic plan rollouts. 
Our method achieves new state-of-the-art performance on benchmarks. This research opens avenues for integrating decision-theoretic principles into LLM system design.

\section*{AI Use Statement} 

Generative AI tools were used solely to improve the presentation of this manuscript, including grammar, syntax, readability, verification of the theoretical results and text and table formatting/refinement. These tools were not used to generate the research ideas, methodology, theoretical contributions, experimental design, results, or scientific conclusions presented in this work. All such contributions originate from the authors. The authors reviewed all AI-assisted edits and take full responsibility for the final content of the manuscript.

\section*{Reproducibility Statement} 
We have taken several steps to facilitate the reproducibility of our results. The paper provides a complete description of the proposed \methodname{} framework, including the span utility formulation, rollout-based dependency estimation procedure, compression algorithm, and defensive verification strategy (Sections~3--4). All theoretical assumptions, derivations, and proofs are provided in Appendix \ref{theoretical_analysis}. Experimental settings, benchmark descriptions, evaluation protocols, model configurations, hyperparameters, and implementation details are described in Section~5 and Appendix \ref{app:implementation_details}. 
Finally, we include the full prompts, hyperparameters, configurations etc. required to reproduce the reported experiments in the Appendix. 
Together, these materials enable independent verification of both the theoretical and empirical results presented in this work.

\bibliography{ICLR27/custom}
\bibliographystyle{ICLR27/iclr2027_conference}

\newpage
\appendix

\section{Appendix} \label{sec:appendix}

\subsection{Theoretical Analysis}
\label{theoretical_analysis}

In this section we provide the theoretical basis for the design choices of
\methodname{}. The guarantees are deliberately \emph{one-sided}: we bound
the decision loss incurred by the spans \methodname{} \emph{discards},
rather than claiming that its retention scores reproduce the exact ranking
of an intractable objective. Concretely, we address the following:

\noindent\textbf{1) Why counterfactual utility is the right objective.}
Proposition~\ref{prop:counterfactual-ib} shows that the expected
counterfactual utility of a span equals its leave-one-out conditional
mutual information with the future, and that, under a bounded
future-interaction condition
(Assumption~\ref{assumption:future-coverage-invariance}), the total
utility of a retained trace matches its predictive information
$I(Z_t;Y\mid g)$ up to an additive error $\bar\varepsilon$. Imposing the
compression term as an entropy budget then makes utility-based selection
the constrained form of the Information Bottleneck objective, with the IB
trade-off coefficient $\beta$ as the inverse budget multiplier.
Corollary~\ref{cor:selection-gap} quantifies the consequence: the trace
selected by utility is within $2\bar\varepsilon$ of the IB-optimal trace in
retained information and in expected decision loss, and coincides with it
whenever the IB optimum is separated by more than $2\bar\varepsilon$. We
give an explicit horizon rate, $\bar\varepsilon/I(Z_t;Y\mid g)=
\mathcal{O}\!\bigl(m(fh+\rho)/K\bigr)$, showing that the approximation
improves with the number $K$ of remaining future events and degrades with
synergy ($f$, $h$) and redundancy ($\rho$) among spans.

\noindent\textbf{2) Why dependency scores are a valid surrogate.}
Exact counterfactual utility requires future-trajectory distributions under
$\mathcal{O}(m)$ counterfactual histories and is undefined for the greedy
policy deployed in practice. \methodname{} instead scores a span by the
probability $u_i$ that the future depends on it. Under a
dependency-directed perturbation model
(Assumption~\ref{ass:dependency-directed}),
Proposition~\ref{prop:utility-decomposition} sandwiches the utility as
$d(u_i\|u_i')\le U(s_i)\le u_i m_i$, and
Corollary~\ref{prop:iff-ranking} turns the upper bound into a safety
certificate: the decision loss of any discarded set $S$ is at most
$m_{\max}\,\mathbb{P}(\text{the future depends on some } s_i\in S)$. This
is exactly what the thresholding rule requires. We state explicitly what
the certificate does \emph{not} imply---agreement between $u$-ranking and
$U$-ranking---and characterize the additional structure (a utility gap, or
a span-independent leak fraction) under which the two rankings coincide
(Remark~\ref{rem:ranking}).

\noindent\textbf{3) How many rollouts suffice.}
The citation frequency $\hat u_i$ is an unbiased Monte Carlo estimate of the
draft model's dependency score; Lemma~\ref{lem:uniform-concentration}
gives a uniform Hoeffding bound with rate $\mathcal{O}(1/\sqrt{N})$, and
Theorem~\ref{thm:high-prob-recovery} shows that, under a utility margin
(Assumption~\ref{assump:utility-margin}), thresholding recovers the
future-relevant set with probability $1-2m\exp(-2N\gamma^2)$. Transferring
these statements from the draft model to the main agent requires a
draft--main coverage condition, $u^{\pi}_S\le\kappa\,u^{q}_S+\epsilon$,
which is the concrete content of the mismatch term
$\delta_{\mathrm{mis}}$ and which we measure per draft model in
Section~\ref{sec:convergence-empirical}. Together these results motivate
the $N$-sweep ablation and explain why $N$ is not an arbitrary
hyperparameter.

\noindent\textbf{4) Why defensive verification helps.}
Proposition~\ref{prop:missed-span-loss} bounds the decision loss by the
summed utility of the spans \methodname{} misses; because
$S_{\mathrm{keep}}=S_U\cup S_R$, adding the verification set can only
shrink this missed set (Theorem~\ref{thm:high-prob-recovery}). Verification
targets precisely the spans for which the certificate in (2) is weakest:
rarely cited but high-impact spans (small $u_i$, large $m_i$) such as
failed calls and state constraints.

The purpose of this section is not to assume that the draft model perfectly
simulates the environment. Instead, we characterize the conditions under
which the \methodname{} estimator is sufficiently accurate to identify
high-utility spans, and how errors in estimation, span interaction, and
draft--main mismatch enter the final decision-preservation guarantee.

Throughout, $Y$ denotes the random \emph{set} of future events (sub-goals,
tool calls, or actions) that the agent will encounter between step $t$ and
termination under goal $g$ (Definition~\ref{def:dependency-score}); all
expectations are over $Y\sim P(\cdot\mid H_t,g)$ and no result requires the
realised future to be known. Recall that the counterfactual future utility
of a span $s_i$ is
\begin{equation}
U(s_i)
=
D\!\left(
P(Y\mid H_t,g)
\;\|\;
P(Y\mid H_t^{-i},g)
\right),
\label{eq:theory-counterfactual-utility}
\end{equation}
where $H_t^{-i}=H_t\setminus\{s_i\}$ and $D$ is taken to be the KL
divergence unless stated otherwise. In practice, \methodname{} does not
compute $U(s_i)$. It estimates the dependency score $u_i$ through
draft-model plan rollouts and uses the empirical estimate $\hat u_i$ (and
its set-level analogue $\hat u_S$) for retention decisions; the remainder
of this appendix analyses that procedure.

\subsection{From Dependency Scores to Counterfactual Utility}
\label{monotone_ranking}

In this section we relate the divergence-based counterfactual utility
$U(s_i)$ to the
dependency score $u_i$ that \methodname{} estimates from Monte Carlo
plan rollouts. The guarantee we establish is deliberately
\emph{one-sided}: a span that is rarely needed by the future has
provably small counterfactual utility, and hence can be dropped at
provably small decision cost. This is exactly the property required by
the thresholding rule in Eq.~\ref{eq:theory-su}. We do not claim, and
the compression rule does not require, that $u_i$ reproduces the
\emph{ordering} of $U(s_i)$ among retained spans; we characterize in
Remark~\ref{rem:ranking} the additional structure under which it does.

\paragraph{Future variable.}
Throughout, $Y$ denotes the random \emph{set} of future events
(sub-goals, tool calls, or actions) that the agent will encounter
between step $t$ and termination under goal $g$; a realisation
$Y=\{y_1,\dots,y_K\}$ is one possible remaining workload. All
expectations below are over $Y\sim P(\cdot\mid H_t,g)$ unless stated
otherwise; no quantity requires the realised future to be known, which
is precisely why $Y$ is integrated out via rollouts.

\begin{definition}[Dependency indicator and dependency score]
\label{def:dependency-score}
For a span $s_i$ and a realisation $Y$ of the future, let
\[
D_i(Y) := \mathbf{1}\{\exists\, y\in Y:\ y \text{ depends on } s_i\}.
\]
The \emph{dependency score} of $s_i$ is the probability that the future
needs it,
\[
u_i := \mathbb{E}_{Y\sim P(\cdot\mid H_t,g)}\bigl[D_i(Y)\bigr]
     = \mathbb{P}\bigl(D_i(Y)=1 \mid H_t,g\bigr).
\]
For a set of spans $S\subseteq H_t$ we write
$D_S(Y):=\mathbf{1}\{\exists\, s_i\in S:\ D_i(Y)=1\}$ and
$u_S:=\mathbb{P}(D_S(Y)=1\mid H_t,g)$. By the union bound,
$u_S\le\sum_{s_i\in S}u_i$.
\end{definition}

\begin{definition}[Signed perturbation]
\label{def:signed-perturbation}
For a span $s_i$ let $H_t^{-i}:=H_t\setminus\{s_i\}$ and define the
signed log-perturbation induced by its removal,
\[
\delta_i(Y) := \log\frac{P(Y\mid H_t,g)}{P(Y\mid H_t^{-i},g)},
\qquad
\text{so that}\qquad
U(s_i)=\mathbb{E}_{Y\sim P(\cdot\mid H_t,g)}\bigl[\delta_i(Y)\bigr].
\]
We further write
\[
u_i' := \mathbb{P}\bigl(D_i(Y)=1 \mid H_t^{-i},g\bigr)
\]
for the probability mass that the counterfactual future places on
$s_i$-dependent realisations.
\end{definition}

\begin{assumption}[Dependency-directed perturbation]
\label{ass:dependency-directed}
Removing a span can only make the futures that depend on it less likely,
and can only make the futures that do not depend on it more likely:
for every realisation $Y$,
\[
\delta_i(Y)\ \ge\ 0 \quad\text{if } D_i(Y)=1,
\qquad\qquad
\delta_i(Y)\ \le\ 0 \quad\text{if } D_i(Y)=0 .
\]
The same is assumed for sets $S\subseteq H_t$ with $D_S$ in place of
$D_i$ and $H_t\setminus S$ in place of $H_t^{-i}$.
\end{assumption}

Assumption~\ref{ass:dependency-directed} formalises the intuition that
deleting $s_i$ drains probability from the trajectories that would have
used it and redistributes that mass onto trajectories that do not. Note
that $\delta_i$ is \emph{signed}: the second clause is not optional but
is forced in aggregate by normalisation, since
$\sum_Y P(Y\mid H_t,g)=\sum_Y P(Y\mid H_t^{-i},g)=1$. An immediate
consequence is $u_i'\le u_i$.

\begin{proposition}[Utility decomposition and sandwich bound]
\label{prop:utility-decomposition}
Under Assumption~\ref{ass:dependency-directed}, the KL-based
counterfactual utility decomposes as
\[
U(s_i) \;=\; u_i\, m_i \;+\; r_i,
\qquad
m_i := \mathbb{E}\bigl[\delta_i(Y)\mid D_i(Y)=1\bigr]\ \ge 0,
\qquad
r_i := \mathbb{E}\bigl[\delta_i(Y)\,\mathbf{1}\{D_i(Y)=0\}\bigr]\ \le 0 .
\]
Moreover, with $d(p\,\|\,q):=p\log\frac{p}{q}+(1-p)\log\frac{1-p}{1-q}$
the binary KL divergence,
\[
d\bigl(u_i\,\|\,u_i'\bigr)\ \le\ U(s_i)\ \le\ u_i\, m_i .
\]
\end{proposition}

\begin{proof}
By the law of total expectation with respect to $D_i(Y)$,
\[
U(s_i)
=\mathbb{P}(D_i{=}1)\,\mathbb{E}[\delta_i\mid D_i{=}1]
+\mathbb{E}[\delta_i\,\mathbf{1}\{D_i{=}0\}]
= u_i m_i + r_i .
\]
The signs follow from Assumption~\ref{ass:dependency-directed}:
$\delta_i\ge0$ on $\{D_i=1\}$ gives $m_i\ge0$, and $\delta_i\le0$ on
$\{D_i=0\}$ gives $r_i\le0$; hence $U(s_i)\le u_i m_i$.
For the lower bound, $D_i$ is a deterministic function of $Y$, so by the
data-processing inequality the divergence between the two future
distributions is at least the divergence between their images under
$D_i$, which are $\mathrm{Bernoulli}(u_i)$ and $\mathrm{Bernoulli}(u_i')$:
$U(s_i)\ge d(u_i\,\|\,u_i')$.
\end{proof}

\begin{remark}[Interpretation of $m_i$]
\label{rem:m-is-kl}
Writing $P_1(\cdot):=P(Y\mid H_t,g,D_i{=}1)$ and
$P_1'(\cdot):=P(Y\mid H_t^{-i},g,D_i{=}1)$ for the two future
distributions restricted to $s_i$-dependent realisations, one has
\[
m_i \;=\; D_{\mathrm{KL}}\bigl(P_1\,\|\,P_1'\bigr) \;+\; \log\frac{u_i}{u_i'} ,
\]
both terms being non-negative. Thus $m_i$ measures how much removing
$s_i$ distorts the futures that actually use it: the first term is the
change in \emph{which} dependent future occurs, the second is the loss
of dependent mass. $m_i$ is small for spans whose content the agent can
cheaply re-derive (e.g.\ an identifier obtainable by one extra call) and
large for spans that gate an irreversible branch (e.g.\ a
failed-payment observation that separates ``stop'' from ``retry'').
\end{remark}

\begin{assumption}[Bounded conditional impact]
\label{ass:impact-scaling}
There is a constant $m_{\max}<\infty$ such that $m_i\le m_{\max}$ for
every span $s_i\in H_t$, and $m_S\le m_{\max}$ for every set
$S\subseteq H_t$ (with $m_S$ defined as in
Proposition~\ref{prop:utility-decomposition} using $D_S$).
\end{assumption}

Assumption~\ref{ass:impact-scaling} is a boundedness condition, not a
monotonicity condition: it does not tie $m_i$ to $u_i$. It holds
whenever the policy has a positive-temperature (full-support) action
distribution, in which case $\delta_i$ is bounded; for the greedy
($T{=}0$) policy used at deployment we take $P(\cdot\mid H_t,g)$ to be
the tempered policy of which the greedy action is the mode, or
equivalently replace $D_{\mathrm{KL}}$ by a bounded divergence such
as total variation, for which $m_{\max}\le 1$ trivially.

\begin{corollary}[One-sided safety certificate]
\label{prop:iff-ranking}
Under Assumptions~\ref{ass:dependency-directed}
and~\ref{ass:impact-scaling}, for every span and every set of spans,
\[
U(s_i)\ \le\ m_{\max}\, u_i,
\qquad\qquad
\]
\[
U(S)\ :=\ D_{\mathrm{KL}}\bigl(P(Y\mid H_t,g)\,\|\,P(Y\mid H_t\setminus S,g)\bigr)
\ \le\ m_{\max}\, u_S\ \le\ m_{\max}\sum_{s_i\in S}u_i .
\]
In particular, the set $S_{\mathrm{drop}}=\{s_i:\ u_i<\tau\}$ removed by
utility thresholding satisfies
$U(S_{\mathrm{drop}})\le m_{\max}\,u_{S_{\mathrm{drop}}}\le m_{\max}\,\tau\,|S_{\mathrm{drop}}|$.
\end{corollary}

\begin{proof}
Apply Proposition~\ref{prop:utility-decomposition} to $s_i$ (resp.\ to
$S$, using the set versions of Assumption~\ref{ass:dependency-directed}
and Definition~\ref{def:dependency-score}), drop the non-positive term
$r$, and bound $m$ by $m_{\max}$. The final inequality is the union
bound $u_S\le\sum_{s_i\in S}u_i$.
\end{proof}

Corollary~\ref{prop:iff-ranking} is the statement that justifies
\methodname{}'s retention rule: the counterfactual decision loss of the
discarded history is controlled by how rarely the future needs it. The
set-level form matters in practice. Two spans that carry the same fact
(e.g.\ an access token that appears in two observations) each have
$u_i\approx0$ individually, since the future needs \emph{one} of them
but not either one specifically; $u_S$ for the pair, however, is the
probability that the future needs the token at all, and it is $u_S$,
not $\sum_i u_i$, that bounds the loss of dropping both. The rollout
estimator below is therefore applied to the candidate dropped set as a
whole (Eq.~\ref{eq:set-estimator}), not only span by span.

\begin{remark}[What the guarantee does not say]
\label{rem:ranking}
Given only $u_i$, Proposition~\ref{prop:utility-decomposition} confines
$U(s_i)$ to the interval $[\,d(u_i\|u_i'),\,u_i m_i\,]$, and these
intervals overlap across spans. Hence $u_i\ge u_j$ does \emph{not} in
general imply $U(s_i)\ge U(s_j)$: a span needed by most futures but
cheaply re-derivable (large $u$, small $m$) can have lower utility than
a rarely needed constraint that redirects the agent (small $u$, large
$m$). Ranking agreement can be stated under additional structure:
\begin{enumerate}
\item[(a)] \emph{Gap condition.} If $m_{\min}\le m_i\le m_{\max}$ and
$|r_i|\le \rho_{\max}$ for all $i$, then
$u_i\,m_{\min}-\rho_{\max}>u_j\,m_{\max}$ implies $U(s_i)>U(s_j)$;
i.e.\ the ordering is preserved for pairs separated by a multiplicative
margin of order $m_{\max}/m_{\min}$.
\item[(b)] \emph{Proportional perturbation.} If removing $s_i$ rescales
every $s_i$-dependent future by the same factor $1-\rho$
(and every non-dependent future by the compensating factor
$\tfrac{1-(1-\rho)u_i}{1-u_i}$), then $u_i'=(1-\rho)u_i$ and the sandwich
of Proposition~\ref{prop:utility-decomposition} collapses to the
equality $U(s_i)=d\bigl(u_i\,\|\,(1-\rho)u_i\bigr)$, which is strictly
increasing in $u_i$ for fixed $\rho\in(0,1)$ (the map $u\mapsto
d(u\|cu)$ is convex, vanishes at $0$, and is positive for $u>0$). Under a
span-independent leak fraction $\rho$, ranking by $u_i$ therefore
coincides with ranking by $U(s_i)$.
\end{enumerate}
Neither condition is needed for the thresholding rule.
 
The defensive verification stage
(Section~\ref{sec_missed_spans}) targets precisely the small-$u$,
large-$m$ spans for which (a) and (b) fail.
\end{remark}

\paragraph{From dependency scores to rollouts.}
\methodname{} does not compute $U(s_i)$ or $u_i$ under the main-agent
policy directly. Each draft-model plan sketch $p_k\sim q_\phi(\cdot\mid
H_t,g)$ is a sampled realisation of the future workload $Y$, and its
citation set $\mathrm{Dep}(p_k)$ is the corresponding realisation of
$\{s_i: D_i(Y)=1\}$. The citation frequency
$\hat u_i=\tfrac1N\sum_k\mathbf{1}[s_i\in\mathrm{Dep}(p_k)]$
(Eq.~\ref{eq:theory-estimator}) is therefore an unbiased Monte Carlo
estimate of the \emph{draft's} dependency score
$u_i^{q}:=\mathbb{P}_{q_\phi}(D_i(Y)=1\mid H_t,g)$, and, for a candidate
dropped set $S$,
\begin{equation}
\hat u_S \;=\; \frac1N\sum_{k=1}^{N}\mathbf{1}\bigl[\mathrm{Dep}(p_k)\cap S\neq\varnothing\bigr]
\label{eq:set-estimator}
\end{equation}
is an unbiased estimate of $u^{q}_S$. Transferring
Corollary~\ref{prop:iff-ranking} from $u^{q}$ to the main-agent score
$u^{\pi}$ requires that the draft cite what the agent will need at
least a constant fraction of the time,
$u^{\pi}_S\le\kappa\,u^{q}_S+\epsilon$; this draft--main coverage
condition is the concrete content of the mismatch term
$\delta_{\mathrm{mis}}$ in Assumption~\ref{assump:residual-error}.
Under it, the decision loss of
the dropped set is bounded by
$m_{\max}\bigl(\kappa\,u^{q}_{S_{\mathrm{drop}}}+\epsilon\bigr)$, and the
estimation error of $\hat u_S$ is controlled by the concentration
results of Section~\ref{sec_utility_separation}.

\subsubsection{Information Bottleneck Objective}\label{ib_proof}

\begin{assumption}[Bounded future-interaction]
\label{assumption:future-coverage-invariance}
For each retained span $s_{i_j}$, let
$P_{i_j}=(s_{i_1},\ldots,s_{i_{j-1}})$ denote the preceding retained
spans and $R_{i_j}=H_t^{-i_j}\setminus P_{i_j}$ the remaining spans of
the observed history. Let $Y$ denote the random set of future events
(Definition~\ref{def:dependency-score}). We assume that, given the
preceding retained spans, the dependence between $s_{i_j}$ and
$R_{i_j}$ is only weakly altered by conditioning on the future:
\[
\varepsilon_{i_j}
:=
\bigl|\,
I(s_{i_j};R_{i_j}\mid Y,P_{i_j},g)
-
I(s_{i_j};R_{i_j}\mid P_{i_j},g)
\,\bigr|
\;\le\;\varepsilon,
\qquad\text{for all } j .
\]
The difference is the (conditional) interaction information
$\mathrm{II}(s_{i_j};R_{i_j};Y\mid P_{i_j},g)$; the assumption therefore
states that the retained span and the rest of the history exhibit
neither \emph{synergy} (a future event that requires $s_{i_j}$
\emph{jointly} with a non-retained span) nor \emph{redundancy}
($s_{i_j}$ duplicating information available in $R_{i_j}$) with respect
to the future, beyond a tolerance $\varepsilon$. It holds exactly
($\varepsilon=0$) when $s_{i_j}\perp\!\!\!\perp R_{i_j}\mid P_{i_j},g$
and the events depending on $s_{i_j}$ depend on no span of $R_{i_j}$.
\end{assumption}

\paragraph{Motivation and horizon dependence.}
Model each future event $y\in Y$ as a function of the spans it depends on
and independent noise, and partition the events into those depending on
$s_{i_j}$ only, on $R_{i_j}$ only, on both, or on neither (given
$P_{i_j}$). Conditioning on additional variables changes a mutual
information by at most their entropy, and conditioning on a function of
one argument cannot increase it, so
\[
\varepsilon_{i_j}\;\le\;
H\bigl(E^{\mathrm{both}}_{i_j}\mid P_{i_j},g\bigr)
+
I(s_{i_j};R_{i_j}\mid P_{i_j},g),
\]
where $E^{\mathrm{both}}_{i_j}$ is the set of future events requiring
both $s_{i_j}$ and some span of $R_{i_j}$. The first term is the
synergy contribution, the second the redundancy contribution. If each
span is needed by at most $f$ future events, each event has bounded
entropy $h$, and redundancy is bounded by $\rho$, then
$\sum_j\varepsilon_{i_j}\le m(fh+\rho)$, whereas the retained predictive
information grows with the number $K$ of remaining events,
$I(Z_t;Y\mid g)=\Theta(K)$. As we shall see in Proposition~\ref{prop:counterfactual-ib}, the relative error of the leave-one-out
decomposition  is 
$O\bigl(m(fh+\rho)/K\bigr)$ and vanishes for long horizons with bounded
fan-in and redundancy. This is the precise sense in which the
assumption is a long-horizon assumption: not that the future is
uninformative about the past, but that any single span's interaction
with the rest of the history is a vanishing share of the total
predictive information. The condition fails, and compression should be
conservative, exactly when one future event depends on many spans at
once.

\begin{proposition}
[Counterfactual utility and the Information Bottleneck objective]
\label{prop:counterfactual-ib}
Let \(H_t=(s_1,\ldots,s_n)\) denote the interaction history, and let
\(Z_t=(s_{i_1},\ldots,s_{i_m})\subseteq H_t\) denote a feasible
compressed trace under a fixed compression budget. Let \(Y\) denote
the random set of future events (Definition~\ref{def:dependency-score}). 

For each retained span \(s_{i_j}\), define
\[
P_{i_j}
=
(s_{i_1},\ldots,s_{i_{j-1}})
\]
as the sequence of retained spans preceding \(s_{i_j}\), and define
\[
R_{i_j}
=
H_t^{-i_j}\setminus P_{i_j}
\]
as the remaining spans in the full history.
Suppose Assumption~\ref{assumption:future-coverage-invariance}
(bounded future-interaction) holds with tolerances
\(\varepsilon_{i_1},\ldots,\varepsilon_{i_m}\), and write
\(\bar\varepsilon:=\sum_{j=1}^{m}\varepsilon_{i_j}\).

Assume that the divergence in
Definition~\ref{eq:counterfactual-utility} is the KL divergence:
\[
U(s_i)
=
D_{\mathrm{KL}}\!\left(
P(Y\mid H_t,g)
\;\middle\|\;
P(Y\mid H_t^{-i},g)
\right).
\]
Then the expected counterfactual utility of a span equals its
leave-one-out conditional mutual information:
\[
\mathbb{E}\bigl[U(s_i)\bigr]
=
I(s_i;Y\mid H_t^{-i},g).
\]

Moreover, for every retained span,                                   
\[
\Bigl|\,
I(s_{i_j};Y\mid H_t^{-i_j},g)
-
I\!\left(
s_{i_j};Y
\mid
s_{i_1},\ldots,s_{i_{j-1}},g
\right)
\Bigr|
\;\le\;\varepsilon_{i_j}.
\]
Consequently,                                                          
\[
\Bigl|\,
I(Z_t;Y\mid g)
-
\sum_{j=1}^{m}
\mathbb{E}\bigl[U(s_{i_j})\bigr]
\Bigr|
\;\le\;\bar\varepsilon ,
\]
with equality \(I(Z_t;Y\mid g)=\sum_{j}\mathbb{E}[U(s_{i_j})]\) when
\(\bar\varepsilon=0\).

Therefore, up to an additive error of \(\bar\varepsilon\), selecting a 
feasible set of spans with the largest total expected counterfactual
utility maximises \(I(Z_t;Y\mid g)\), and—imposing the compression
term as a budget constraint as shown below—this selection is the
constrained form of the Information Bottleneck objective
\[
\min_C\;
I(H_t;Z_t\mid g)
-
\beta I(Z_t;Y\mid g).
\]
\end{proposition}

\begin{proof}
We first establish the connection between counterfactual utility and
conditional mutual information. For a span \(s_i\), let
\(H_t^{-i}=H_t\setminus\{s_i\}\). By definition,
\[
U(s_i)
=
D_{\mathrm{KL}}\!\left(
P(Y\mid H_t,g)
\;\middle\|\;
P(Y\mid H_t^{-i},g)
\right).
\]
Taking the expectation over the joint distribution of
\((H_t,Y,g)\), we obtain
\[
\begin{aligned}
\mathbb{E}[U(s_i)]
&=
\mathbb{E}_{H_t,Y,g}
\left[
\log
\frac{P(Y\mid H_t,g)}
     {P(Y\mid H_t^{-i},g)}
\right] \\
&=
\mathbb{E}_{H_t,Y,g}
\left[
\log
\frac{P(Y\mid s_i,H_t^{-i},g)}
     {P(Y\mid H_t^{-i},g)}
\right].
\end{aligned}
\]
Since \(H_t=(H_t^{-i},s_i)\), this is exactly the definition of
conditional mutual information:
\[
\mathbb{E}[U(s_i)]
=
I(s_i;Y\mid H_t^{-i},g).
\]

We next show that, under Assumption~\ref{assumption:future-coverage-invariance}, 
this leave-one-out quantity is within \(\varepsilon_{i_j}\) of the
corresponding term in the chain-rule decomposition of the retained trace.

For a retained span \(s_{i_j}\), write
\[
H_t^{-i_j}
=
(P_{i_j},R_{i_j}),
\]
where
\[
P_{i_j}
=
(s_{i_1},\ldots,s_{i_{j-1}})
\]
contains the preceding retained spans and \(R_{i_j}\) contains the
remaining spans.

By the chain rule of mutual information,
\[
\begin{aligned}
I(s_{i_j};Y,R_{i_j}\mid P_{i_j},g)
&=
I(s_{i_j};R_{i_j}\mid P_{i_j},g) \\
&\quad+
I(s_{i_j};Y\mid P_{i_j},R_{i_j},g).
\end{aligned}
\]
Applying the same chain rule in the opposite order yields
\[
\begin{aligned}
I(s_{i_j};Y,R_{i_j}\mid P_{i_j},g)
&=
I(s_{i_j};Y\mid P_{i_j},g) \\
&\quad+
I(s_{i_j};R_{i_j}\mid Y,P_{i_j},g).
\end{aligned}
\]
Equating these two expressions gives
\[
\begin{aligned}
I(s_{i_j};Y\mid P_{i_j},R_{i_j},g)
&=
I(s_{i_j};Y\mid P_{i_j},g) \\
&\quad+
I(s_{i_j};R_{i_j}\mid Y,P_{i_j},g) \\
&\quad-
I(s_{i_j};R_{i_j}\mid P_{i_j},g).
\end{aligned}
\]
The last two terms are exactly the quantity controlled by
Assumption~\ref{assumption:future-coverage-invariance}:
\[
\Bigl|
I(s_{i_j};R_{i_j}\mid Y,P_{i_j},g)
-
I(s_{i_j};R_{i_j}\mid P_{i_j},g)
\Bigr|
\;\le\;\varepsilon_{i_j}.
\]
Hence, since \(H_t^{-i_j}=(P_{i_j},R_{i_j})\),
\[
\Bigl|
I(s_{i_j};Y\mid H_t^{-i_j},g)
-
I(s_{i_j};Y\mid P_{i_j},g)
\Bigr|
\;\le\;\varepsilon_{i_j},
\]
that is,
\[
\Bigl|
I(s_{i_j};Y\mid H_t^{-i_j},g)
-
I\!\left(
s_{i_j};Y
\mid
s_{i_1},\ldots,s_{i_{j-1}},g
\right)
\Bigr|
\;\le\;\varepsilon_{i_j}.
\]
The bound is tight in the sense that it is attained with equality
when \(\varepsilon_{i_j}=0\), i.e.\ when \(s_{i_j}\) and \(R_{i_j}\)
have no interaction with respect to the future.

Combining this with the relationship between counterfactual utility
and conditional mutual information gives                               
\[
\Bigl|
\mathbb{E}[U(s_{i_j})]
-
I\!\left(
s_{i_j};Y
\mid
s_{i_1},\ldots,s_{i_{j-1}},g
\right)
\Bigr|
\;\le\;\varepsilon_{i_j}.
\]

Now apply the standard chain rule of mutual information to the
retained trace \(Z_t=(s_{i_1},\ldots,s_{i_m})\) and sum the per-span
bounds:                                                                
\[
\begin{aligned}
\Bigl|\,
I(Z_t;Y\mid g)
-
\sum_{j=1}^{m}\mathbb{E}[U(s_{i_j})]
\Bigr|
&=
\Bigl|\,
\sum_{j=1}^{m}
\Bigl[
I\!\left(
s_{i_j};Y
\mid
s_{i_1},\ldots,s_{i_{j-1}},g
\right)
-
I(s_{i_j};Y\mid H_t^{-i_j},g)
\Bigr]
\Bigr| \\
&\le
\sum_{j=1}^{m}\varepsilon_{i_j}
\;=\;\bar\varepsilon .
\end{aligned}
\]
Thus, under Assumption~\ref{assumption:future-coverage-invariance}, the
predictive information retained by \(Z_t\) equals the sum of the
expected counterfactual utilities of its retained spans up to an
additive error \(\bar\varepsilon\), which vanishes when no retained
span interacts with the remaining history with respect to the future.
By the horizon argument accompanying the assumption,
\(\bar\varepsilon\le m(fh+\rho)\) while \(I(Z_t;Y\mid g)=\Theta(K)\),
so the relative error is \(O\bigl(m(fh+\rho)/K\bigr)\) for long
horizons.

Finally, consider the Information Bottleneck objective
\[
\mathcal{L}_{\mathrm{IB}}(C)
=
I(H_t;Z_t\mid g)
-
\beta\, I(Z_t;Y\mid g),
\qquad \beta>0 .
\]
Because the compressor $C$ is a deterministic span-selection map,
$Z_t$ is a function of $H_t$ and the compression term reduces to the
entropy of the retained trace,
\[
I(H_t;Z_t\mid g)=H(Z_t\mid g)-H(Z_t\mid H_t,g)=H(Z_t\mid g).
\]
Rather than treating this term as constant, we impose it as a
\emph{budget}: a compressed trace is feasible if it retains at most
$B$ nats of the history,
\[
\mathcal{Z}(B):=\bigl\{\,Z_t\subseteq H_t:\ H(Z_t\mid g)\le B\,\bigr\}.
\]
Writing $\ell_i$ for the token length of span
$s_i$ under the agent's tokenizer, and $\bar h$ for the per-token
code-length bound of the model, $H(Z_t\mid g)\le \bar h\sum_{s_i\in Z_t}\ell_i$,
so the token constraint $\sum_{s_i\in Z_t}\ell_i\le B_{\mathrm{tok}}$
implies the entropy constraint with $B=\bar h\,B_{\mathrm{tok}}$.

The constrained decision-preservation problem is therefore
\begin{equation}
\max_{Z_t\in\mathcal{Z}(B)}\; I(Z_t;Y\mid g)
\;\;=\;\;
\max_{Z_t\subseteq H_t}\; I(Z_t;Y\mid g)
\quad\text{s.t.}\quad H(Z_t\mid g)\le B .
\label{eq:ib-constrained}
\end{equation}
By the decomposition established above,
$I(Z_t;Y\mid g)=\sum_{s_i\in Z_t}\mathbb{E}[U(s_i)]$ (up to the
approximation error $\sum_j\varepsilon_{i_j}$ of the approximation form),
so Eq.~\ref{eq:ib-constrained} reads
\begin{equation}
\max_{Z_t\subseteq H_t}\;\sum_{s_i\in Z_t}\mathbb{E}[U(s_i)]
\quad\text{s.t.}\quad H(Z_t\mid g)\le B .
\label{eq:utility-knapsack}
\end{equation}
This is a budgeted selection (knapsack) problem in which each span
contributes its expected counterfactual utility to the objective and its
code length to the constraint.

Introducing a Lagrange multiplier $\lambda\ge 0$ for the budget gives
the Lagrangian
\[
\mathcal{L}(Z_t,\lambda)
=
\sum_{s_i\in Z_t}\mathbb{E}[U(s_i)]
-\lambda\bigl(H(Z_t\mid g)-B\bigr)
=
I(Z_t;Y\mid g)-\lambda\, I(H_t;Z_t\mid g)+\lambda B .
\]
Dividing by $\lambda>0$ and dropping the constant $\lambda B$, maximising
$\mathcal{L}(\cdot,\lambda)$ over $Z_t$ is equivalent to
\[
\min_{Z_t\subseteq H_t}\;
I(H_t;Z_t\mid g)-\beta\, I(Z_t;Y\mid g),
\qquad \beta:=1/\lambda ,
\]
which is exactly $\mathcal{L}_{\mathrm{IB}}(C)$. Hence, for every budget
$B$ there is a trade-off coefficient $\beta=1/\lambda^\star(B)$ such that
the budgeted utility-selection problem \eqref{eq:utility-knapsack} and
the IB objective share the same Lagrangian, and the IB solution at
$\beta$ is budget-feasible for $B=I(H_t;Z_t^\star\mid g)$; conversely,
sweeping $\beta$ traces the Lagrangian relaxation of the constraint set
$\mathcal{Z}(B)$ as $B$ varies. (Since span selection is combinatorial
the relaxation may leave a duality gap; the solutions coincide exactly at
budgets attained by a support point of the concave envelope of the
utility--entropy frontier, and elsewhere the IB solution is the
utility-maximising feasible trace for the nearest such budget.)

Under the additional simplification that all spans have equal code
length, $H(Z_t\mid g)\propto|Z_t|$, the constraint becomes a cardinality
constraint, the knapsack reduces to top-$k$ selection, and
\[
\arg\max_{|Z_t|\le k}\;\sum_{s_i\in Z_t}\mathbb{E}[U(s_i)]
=
\arg\max_{|Z_t|\le k}\; I(Z_t;Y\mid g)
=
\arg\min_{C:\,|Z_t|\le k}\;\mathcal{L}_{\mathrm{IB}}(C),
\]
recovering the fixed-cardinality statement as a special case.

Therefore, under the conditional-independence (bounded-interaction)
assumption, selecting spans by expected counterfactual utility subject
to the entropy budget $H(Z_t\mid g)\le B$ is the constrained form of the
Information Bottleneck objective, with the IB trade-off coefficient
$\beta$ playing the role of the inverse Lagrange multiplier of the
budget.
\end{proof}

\begin{corollary}[Gap between utility-selected and IB-optimal traces]
\label{cor:selection-gap}
Let $\mathcal{Z}(B)$ be the budget-feasible traces, let
$Z^\star\in\arg\max_{Z\in\mathcal{Z}(B)} I(Z;Y\mid g)$ be the IB-optimal trace,
and let $\hat Z\in\arg\max_{Z\in\mathcal{Z}(B)}\sum_{s_i\in Z}\mathbb{E}[U(s_i)]$
be the trace selected by total counterfactual utility. Under
Assumption~\ref{assumption:future-coverage-invariance},
\begin{enumerate}
\item[(i)] \emph{Information gap.}\;
$0\;\le\; I(Z^\star;Y\mid g)-I(\hat Z;Y\mid g)\;\le\;2\bar\varepsilon$;
with $I(Z^\star;Y\mid g)=\Theta(K)$ the relative gap is $O\!\bigl(m(fh+\rho)/K\bigr)$.
\item[(ii)] \emph{Decision-distribution gap.}\;
Writing $L(Z):=\mathbb{E}_{H_t}\bigl[D_{\mathrm{KL}}(P(Y\mid H_t,g)\,\|\,P(Y\mid Z,g))\bigr]
= I(H_t;Y\mid g)-I(Z;Y\mid g)$ for the expected decision loss of a trace,
\[
L(\hat Z)\;\le\;L(Z^\star)+2\bar\varepsilon,
\qquad
\mathbb{E}_{H_t}\bigl[\mathrm{TV}\bigl(P(Y\mid H_t,g),P(Y\mid\hat Z,g)\bigr)\bigr]
\;\le\;\sqrt{\tfrac12\bigl(L(Z^\star)+2\bar\varepsilon\bigr)} ,
\]
and the same bounds hold for the marginal of any future action $a_{t+k}$,
since it is a function of $Y$.
\item[(iii)] \emph{Exact agreement under a margin.}\;
If the IB-optimal trace is separated from every other feasible trace by
$I(Z^\star;Y\mid g)-I(Z;Y\mid g)>2\bar\varepsilon$ for all
$Z\in\mathcal{Z}(B)\setminus\{Z^\star\}$, then $\hat Z=Z^\star$: the two
optimizations select identical spans.
\end{enumerate}
\end{corollary}

\begin{proof}
Let $f(Z):=I(Z;Y\mid g)$ and $u(Z):=\sum_{s_i\in Z}\mathbb{E}[U(s_i)]$, so
$|f-u|\le\bar\varepsilon$ on $\mathcal{Z}(B)$ by
Proposition~\ref{prop:counterfactual-ib}. Then
$f(\hat Z)\ge u(\hat Z)-\bar\varepsilon\ge u(Z^\star)-\bar\varepsilon
\ge f(Z^\star)-2\bar\varepsilon$, giving (i). For (ii),
$L(Z)=I(H_t;Y\mid g)-f(Z)$ is the identity of Section~3 (the expected KL
equals the entropy difference), so $L(\hat Z)-L(Z^\star)=f(Z^\star)-f(\hat Z)\le2\bar\varepsilon$;
the TV bound is Pinsker's inequality followed by Jensen, and the
action-level statement is data processing. For (iii), if $\hat Z\neq Z^\star$
then by (i) $f(\hat Z)\ge f(Z^\star)-2\bar\varepsilon$, contradicting the margin.
\end{proof}

In practice, \methodname{} does not compute the KL divergence in Equation~\ref{eq:counterfactual-utility} directly. Instead, it approximates future utility using stochastic draft-model plan sketches, where a span is scored by the fraction of sampled plans that explicitly cite it as a dependency. This citation frequency can be interpreted as a Monte Carlo estimate of the probability that the span is decision-relevant for the future. 
From the above analysis, this estimate is within the bounds connected to the expected KL-based counterfactual utility, and thus to the solution optimizing the information bottleneck objective.

\subsubsection{Utility Separation and Estimation Accuracy}\label{sec_utility_separation}

We first state a margin condition under which future-relevant and irrelevant spans are separable by a threshold. Let $S^*$ denote the set of future-relevant spans that should be preserved for decision-making.

\begin{assumption}[Utility Margin]
\label{assump:utility-margin}
There exists a threshold $\tau$ and margin $\gamma>0$ such that every future-relevant span $s_i\in S^*$ satisfies
\begin{equation*}
u_i \geq \tau+\gamma,
\label{eq:margin-positive}
\end{equation*}
and every irrelevant span $s_i\notin S^*$ satisfies
\begin{equation*}
u_i \leq \tau-\gamma.
\label{eq:margin-negative}
\end{equation*}
\end{assumption}

This assumption states that useful and non-useful spans are not arbitrarily close under the utility measure. Without such a margin, any threshold-based compression method may be unstable, since a small estimation error could flip the retention decision.

\begin{figure}[t]
    \centering
    \includegraphics[width=\linewidth]{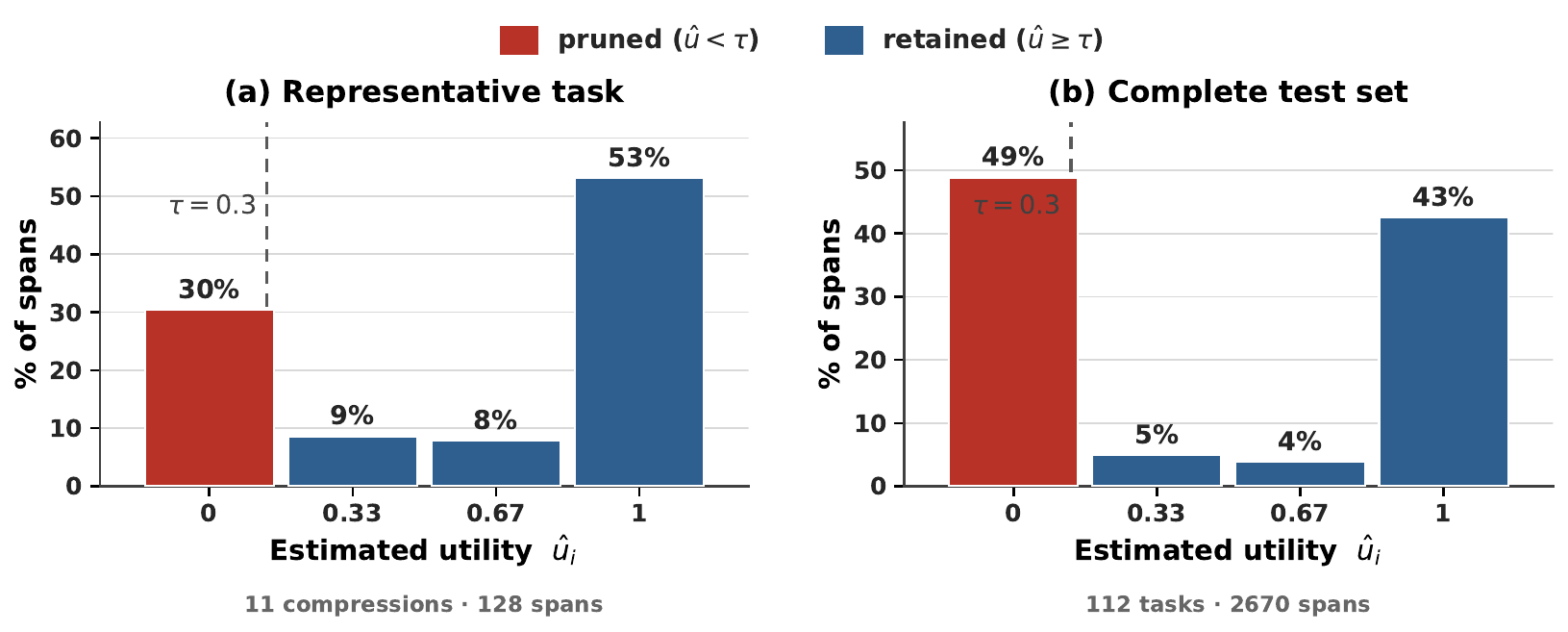}
    \caption{Empirical distribution of estimated span utilities $\hat{u}_i$ on AppWorld with \texttt{gpt-4.1} as both agent and draft model with $N{=}3$ rollouts (figures rounded off to the next digit). \textbf{(a)} A representative task (\texttt{task\_6b6ca61\_3}) and \textbf{(b)} the complete \texttt{test\_normal} set. Bars are colored by the retention threshold $\tau{=}0.3$ (red: pruned, blue: retained). Almost all probability mass concentrates at the extremes $\hat{u}_i\!\approx\!0$ and $\hat{u}_i\!\approx\!1$ ($\sim83\%$ for the representative task, $\sim92\%$ in aggregate), leaving a sparsely populated interior near $\tau$. This empirically supports the utility-margin separation posited by Assumption~\ref{assump:utility-margin}.} 
    \label{fig:utility_histogram}
\end{figure}

\paragraph{Empirical validation of the margin.} Figure~\ref{fig:utility_histogram} plots the estimated span utilities $\hat{u}_i$ produced during actual compression on AppWorld. The distribution is sharply concentrated at the two extremes: spans that no plan sketch cites ($\hat{u}_i{=}0$, irrelevant) and spans cited by every rollout ($\hat{u}_i{=}1$, decision-critical), with only $8\text{--}16\%$ of the mass in the interior between them. Very little mass falls in the immediate neighborhood of the threshold $\tau$ (under $10\%$ within $\pm 1/6$ of $\tau$ in aggregate), so the retention decision is rarely a close call. This separation is precisely the well-separated, near-bimodal structure over the true utilities $u_i$ that Assumption~\ref{assump:utility-margin} formalizes, indicating that the margin condition is not merely a convenient idealization but is approximately realized in practice. We note that, at the default $N{=}3$, the estimator $\hat{u}_i$ is discrete on $\{0,\tfrac{1}{3},\tfrac{2}{3},1\}$, so the visible gap reflects the separation of the underlying utilities rather than an artifact of high-resolution binning.

Next, we connect the empirical estimator $\hat{u}_i$ to the draft-model rollout process. For each rollout $k$, define
\begin{equation}
X_i^{(k)}
=
\mathbf{1}\!\left[s_i\in \mathrm{Dep}(p_k)\right],
\label{eq:theory-dependency-indicator}
\end{equation}
where $p_k$ is the sampled plan sketch and $\mathrm{Dep}(p_k)$ is the set of spans cited by that plan. \methodname{} estimates the dependency score as
\begin{equation}
\hat{u}_i
=
\frac{1}{N}\sum_{k=1}^N X_i^{(k)}.
\label{eq:theory-estimator}
\end{equation}

\begin{lemma}[Uniform Concentration of Utility Estimates]
\label{lem:uniform-concentration}
Assume that, for each span $s_i$, the rollout indicators $X_i^{(k)}$ are independent Bernoulli random variables with expectation $u_i$. Let $m=|H_t|$ be the number of spans. Then, for any $\gamma>0$,
\begin{equation}
\Pr\!\left(
\max_{1\leq i\leq m} |\hat{u}_i-u_i|>\gamma
\right)
\leq
2m\exp(-2N\gamma^2).
\label{eq:uniform-concentration}
\end{equation}
\end{lemma}

\begin{proof}[Proof sketch]
For a fixed span $s_i$, Hoeffding's inequality gives
\begin{equation*}
\Pr\!\left(|\hat{u}_i-u_i|>\gamma\right)
\leq
2\exp(-2N\gamma^2).
\end{equation*}
Applying a union bound over all $m$ spans yields Equation~\ref{eq:uniform-concentration}.
\end{proof}

Lemma~\ref{lem:uniform-concentration} shows that increasing the number of draft-model rollouts reduces the probability of any span-level utility estimate deviating from its population value. This provides a formal justification for using multiple stochastic plan sketches rather than a single deterministic plan.

\subsubsection{Recovery of Future-Relevant Spans}

\methodname{} first retains spans whose estimated utility exceeds the threshold:
\begin{equation}
S_U
=
\{s_i\in H_t:\hat{u}_i\geq \tau\}.
\label{eq:theory-su}
\end{equation}
It then adds spans rescued by defensive verification:
\begin{equation}
S_{\mathrm{keep}}
=
S_U\cup S_R,
\label{eq:theory-skeep}
\end{equation}
where $S_R$ contains spans whose deletion may cause repeated invalid actions, state inconsistency, or loss of negative constraints.

\begin{theorem}[High-Probability Recovery of Future-Relevant Spans]
\label{thm:high-prob-recovery}
Suppose Assumption~\ref{assump:utility-margin} holds and the rollout indicators satisfy the conditions of Lemma~\ref{lem:uniform-concentration}. Then, with probability at least
\begin{equation}
1-2m\exp(-2N\gamma^2),
\label{eq:recovery-probability}
\end{equation}
thresholding $\hat{u}_i$ at $\tau$ recovers all future-relevant spans in $S^*$ and excludes all irrelevant spans outside the margin. That is, $S^* \subseteq S_U$. Moreover, the final \methodname{} retained set satisfies $S^* \subseteq S_{\mathrm{keep}}$, since $S_{\mathrm{keep}}=S_U\cup S_R$.
\end{theorem}

\begin{proof}[Proof sketch]
By Lemma~\ref{lem:uniform-concentration}, with probability at least $1-2m\exp(-2N\gamma^2)$, all spans satisfy $|\hat{u}_i-u_i|\leq\gamma$. For any future-relevant span $s_i\in S^*$, Assumption~\ref{assump:utility-margin} gives $u_i\geq\tau+\gamma$, hence
\begin{equation*}
\hat{u}_i \geq u_i-\gamma \geq \tau.
\end{equation*}
Therefore $s_i\in S_U$. For any irrelevant span $s_j\notin S^*$, Assumption~\ref{assump:utility-margin} gives $u_j\leq\tau-\gamma$, hence
\begin{equation*}
\hat{u}_j \leq u_j+\gamma \leq \tau.
\end{equation*}
Thus irrelevant spans below the margin are excluded by utility thresholding. Finally, since $S_{\mathrm{keep}}=S_U\cup S_R$, adding defensive verification cannot remove any recovered span, so $S^*\subseteq S_{\mathrm{keep}}$.
\end{proof}

This theorem gives the main recovery guarantee. Under a utility margin, the probability of recovering the future-relevant span set increases with the number of rollouts $N$ and decreases with the number of spans $m$. Defensive verification then acts as a conservative safety extension: it may increase the retained set, but it does not reduce the recovered high-utility set.

\subsubsection{From Span Recovery to Decision Preservation}

We next connect span recovery to decision preservation. Let the compressed trace be
\begin{equation}
Z_t
=
\{C(s_i):s_i\in H_t,\ C(s_i)\neq\varnothing\},
\label{eq:theory-compressed-trace}
\end{equation}
where
\begin{equation*}
C(s_i)=
\begin{cases}
s_i, & s_i\in S_{\mathrm{keep}},\\
\varnothing, & s_i\in S_{\mathrm{drop}}.
\end{cases}
\label{eq:theory-compression-map}
\end{equation*}

To state a decision-preservation result, we separate the loss introduced by dropping spans from the loss introduced 
by draft-main model mismatch.
\begin{assumption}[Bounded Residual Compression Error]
\label{assump:residual-error}
Conditioned on retaining all future-relevant spans in $S^*$, the remaining divergence between the full-history future distribution and the compressed-trace future distribution is bounded by $\delta_{\mathrm{res}}\equiv \delta_{\mathrm{mis}}$,

which captures mismatch between the draft-model dependency structure and the main-agent policy.
\end{assumption}

Assumption~\ref{assump:residual-error} reflects the fact that even if \methodname{} correctly identifies which spans should be preserved, the final compressed trace may still differ from the full history because 
the draft model may not perfectly match the main agent.

\begin{corollary}[Decision Preservation under Span Recovery]
\label{cor:decision-preservation}
Under the conditions of Theorem~\ref{thm:high-prob-recovery} and Assumption~\ref{assump:residual-error}, with probability at least
\begin{equation}
1-2m\exp(-2N\gamma^2),
\end{equation}
the compressed trace $Z_t$ approximately preserves the future decision distribution:
\begin{equation}
D\!\left(
P(Y\mid H_t,g)
\;\|\;
P(Y\mid Z_t,g)
\right)
\leq
\delta_{\mathrm{res}}.
\label{eq:decision-preservation-bound}
\end{equation}
\end{corollary}

\begin{proof}[Proof sketch]
By Theorem~\ref{thm:high-prob-recovery}, with high probability \methodname{} retains all spans in $S^*$. Therefore, no future-relevant span is removed. The remaining discrepancy between conditioning on $H_t$ and conditioning on $Z_t$ comes from
the draft-main mismatch. By Assumption~\ref{assump:residual-error}, these residual errors are bounded by $\delta_{\mathrm{res}}$.
\end{proof}

This corollary states that \methodname{} approximates a decision-preserving sufficient statistic of the full history: if the retained set covers the future-relevant spans, then the compressed trace preserves future behaviour up to residual compression and modeling errors.

\subsubsection{Effect of Missed Spans}\label{sec_missed_spans}

The previous result assumes that all future-relevant spans are recovered. We also state a more general bound that accounts for missed spans. Let
\begin{equation}
M
=
S^*\setminus S_{\mathrm{keep}}
\label{eq:missed-set}
\end{equation}
denote the set of future-relevant spans missed by \methodname{}.

\begin{assumption}[Subadditive Omission Loss]
\label{assump:omission-loss}
The additional future-distribution divergence caused by omitting a set of future-relevant spans $M$ is bounded by the sum of their utilities:
\begin{equation}
D\!\left(
P(Y\mid H_t,g)
\;\|\;
P(Y\mid Z_t,g)
\right)
\leq
\delta_{\mathrm{res}}
+
\sum_{s_i\in M} U(s_i).
\label{eq:subadditive-omission-loss}
\end{equation}
\end{assumption}

This assumption is a conservative way to account for missed causal information. It does not require \methodname{} to be perfect; instead, it states that the decision loss increases with the total counterfactual utility of the missed spans.

\begin{proposition}[Decision Loss with Missed Spans]
\label{prop:missed-span-loss}
If \methodname{} misses a set $M=S^*\setminus S_{\mathrm{keep}}$ of future-relevant spans, then under Assumption~\ref{assump:omission-loss},
\begin{equation}
D\!\left(
P(Y\mid H_t,g)
\;\|\;
P(Y\mid Z_t,g)
\right)
\leq
\delta_{\mathrm{res}}
+
\sum_{s_i\in M} U(s_i).
\label{eq:missed-span-bound}
\end{equation}
\end{proposition}

Proposition~\ref{prop:missed-span-loss} clarifies the failure mode of \methodname{}: the main source of decision degradation is not compression itself, but the omission of spans with non-negligible counterfactual future utility. This directly motivates both Monte Carlo rollouts, which reduce the probability of missing high-utility spans, and defensive verification, which rescues spans that may not be frequently cited but encode important negative constraints.

\noindent \textbf{When Does Defensive Verification Help? }
We observe the benefit of defensive verification (FOCUS-D) is uneven across benchmarks, and this asymmetry is informative rather than a weakness. On AppWorld, FOCUS-D lifts accuracy substantially over the optimistic variant (56.5\% $\rightarrow$ 64.9\%), whereas on OfficeBench and 8-objective QA it trails FOCUS-O by only $\sim$1 point (78.9\% $\rightarrow$ 77.9\% and 0.386 $\rightarrow$ 0.376). This $\sim$1-point gap is within the run-to-run variance of LLM API non-determinism on $\sim$100-task splits ($\pm 1.4\%$ std., Section~\ref{sec:main_results}) and should not be read as systematic degradation. By construction, defensive verification \emph{rescues} spans that optimistic selection would otherwise drop.
Its sole cost is additional token usage from retaining a few extra spans, i.e.\ a compute-robustness trade-off. The stage is therefore most valuable precisely when the draft model is overconfident in discarding a span, or when the environment is stateful/stochastic and hard to simulate (as in AppWorld's API-driven tasks). Conversely, when the environment is largely deterministic and predictable, defensive verification may retain spans that are ultimately unneeded, adding mild context noise without accuracy benefit, which explains the flat-to-slightly-lower numbers on OfficeBench and 8-QA, which is admissible within variance due to LLM stochasticity.

\subsubsection{Empirical Validation of the Draft-Model Estimator}
\label{sec:convergence-empirical}
Because causal span relevance has no directly observable ground-truth label, conventional precision and recall cannot be computed without imposing an additional annotation-based proxy; we therefore evaluate the estimator against a leave-one-span-out counterfactual reference and report ranking and retained-set agreement. We empirically test the two predictions our theory makes about the draft-rollout estimator $\hat{u}_i$: it concentrates uniformly across spans at rate $O(1/\sqrt{N})$ (Lemma~\ref{lem:uniform-concentration}), and thresholding it recovers the correct keep-set with high probability (Theorem~\ref{thm:high-prob-recovery}). We further
contrast the draft estimator against the intractable KL counterfactual utility
(Eq.~\ref{eq:counterfactual-utility}) it is designed to approximate.

\paragraph{Protocol.} We reconstruct $30$ real mid-trajectory compression events (each with $\ge 6$ history spans) from saved OfficeBench and Appworld trajectories each. At each event we score every span two ways: (i) the \emph{draft} citation-frequency estimator used by \methodname{}, and (ii) the \emph{KL counterfactual} utility, estimated by leave-one-span-out sampling over an LLM-elicited future-action vocabulary.

\begin{figure}[t]
    \centering
    \includegraphics[width=\linewidth]{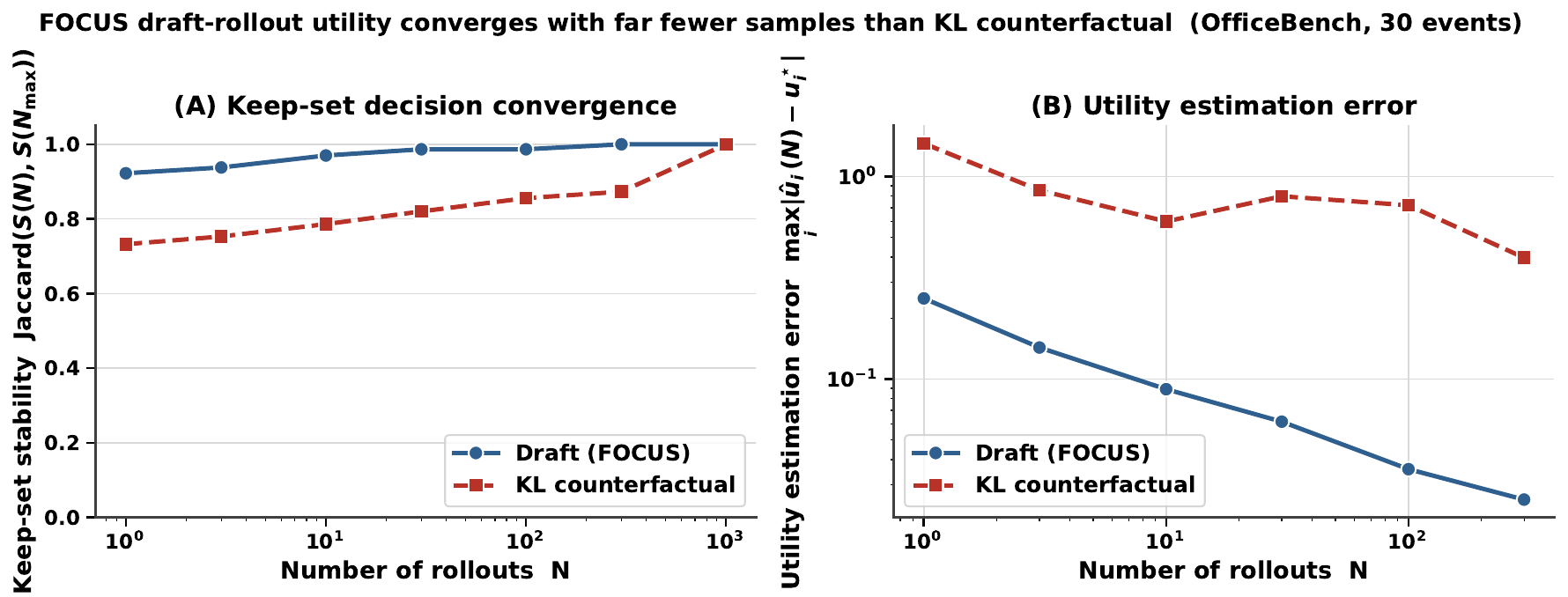}
    \caption{Sample efficiency of the draft-rollout estimator vs.\ direct KL counterfactual
estimation ($30$ OfficeBench compression events). \textbf{(A)} Keep-set stability, i.e.\
Jaccard overlap between the decision at $N$ and each estimator's own $N_{\max}$ decision.
\textbf{(B)} Worst-span utility estimation error $\max_i|\hat{u}_i(N)-u^\star_i|$ (log-log).
The draft estimator locks its decision by $N\approx10$ and its error decays as $O(1/\sqrt{N})$,
whereas KL needs $\sim\!100\times$ more rollouts and stays $7$--$15\times$ noisier. This indicates that estimating counterfactual utility using KL divergence is intractable in practice.}
    \label{fig:draft_convergence}
\end{figure}

\paragraph{Estimation error decays as $O(1/\sqrt{N})$ (Lemma~\ref{lem:uniform-concentration}).} Figure~\ref{fig:draft_convergence}(B) reports the worst-span error $\max_i|\hat{u}_i(N)-u^\star_i|$. The draft error falls monotonically
($0.25\!\to\!0.089\!\to\!0.025$ at $N{=}1,10,300$; log--log slope $\approx-0.40$),
approaching the $O(1/\sqrt{N})$ rate predicted by Lemma~\ref{lem:uniform-concentration} (ideal slope $-0.5$), whereas the KL counterfactual error stays
$7$--$15\times$ larger and barely concentrates (slope $\approx-0.17$).

\paragraph{The keep-decision locks with far fewer samples (Theorem~\ref{thm:high-prob-recovery}).}
Theorem~\ref{thm:high-prob-recovery} predicts the retained set stabilizes exponentially
fast in $N$ (rate $1-2m\exp(-2N\gamma^2)$). Figure~\ref{fig:draft_convergence}(A) confirms
this: the draft keep-set matches its converged $N_{\max}$ decision at Jaccard $\approx0.92$
after a \emph{single} rollout and locks by $N\approx10$, whereas KL needs $N\approx1000$
for comparable self-agreement. This $\sim\!100\times$ sample-efficiency gap motivates
\methodname{}'s use of draft rollouts and justifies our small default of $N{=}3$.

\paragraph{Agreement with the KL counterfactual reference.}
Fast convergence establishes that the draft estimator is low-variance, but not that it agrees with the intractable KL counterfactual utility it approximates. We therefore compare the two estimators directly at their converged $N_{\max}$ decisions on the same $30$ compression events, using two complementary measures: the \emph{overlap coefficient} between their retained span sets, and the per-event \emph{Spearman} rank correlation of
their keep/drop decisions.
On both benchmarks the two estimators select highly overlapping retained sets ($0.85$ on OfficeBench, $0.90$ on AppWorld), and their decisions are positively rank-correlated ($0.39$ and $0.48$ respectively). The agreement is stronger on AppWorld, whose longer horizons yield less noisy KL references compared to OfficeBench's shorter trajectories.

\subsubsection{Discussion}

The analysis above yields three implications. First, \methodname{} is most reliable when future-relevant spans are utility-separated from irrelevant spans. Second, the number of rollouts $N$ provides a direct statistical trade-off: larger $N$ improves span recovery but increases compression-time compute. Third, defensive verification provides a conservative correction mechanism for stateful environments, where failed actions and negative observations may be crucial even if they are not frequently cited in optimistic future plans.

Overall, the theoretical guarantee can be summarized as follows: under a utility margin and bounded estimation error, \methodname{} recovers high-utility spans with high probability; when those spans are retained, the compressed trace preserves the future decision distribution up to residual 
draft-main mismatch errors.

\begin{table*}[ht]
\centering

\footnotesize
\setlength{\tabcolsep}{1.8pt}
\renewcommand{\arraystretch}{0.9}

\begin{adjustbox}{max width=\textwidth}
\begin{tabular}{lccccccccccccc}
\toprule

& \multicolumn{4}{c}{Average (168)}
& \multicolumn{3}{c}{Easy (57)}
& \multicolumn{3}{c}{Medium (48)}
& \multicolumn{3}{c}{Hard (63)} \\

\cmidrule(lr){2-5}
\cmidrule(lr){6-8}
\cmidrule(lr){9-11}
\cmidrule(lr){12-14}

\textbf{Method}
& \textbf{Acc.$\uparrow$}
& \textbf{Steps$\downarrow$}
& \textbf{Peak$\downarrow$}
& \textbf{Dep.$\downarrow$}

& \textbf{Acc.$\uparrow$}
& \textbf{Peak$\downarrow$}
& \textbf{Dep.$\downarrow$}

& \textbf{Acc.$\uparrow$}
& \textbf{Peak$\downarrow$}
& \textbf{Dep.$\downarrow$}

& \textbf{Acc.$\uparrow$}
& \textbf{Peak$\downarrow$}
& \textbf{Dep.$\downarrow$}
\\

\midrule

\multicolumn{14}{c}{
\textbf{Agent:} \texttt{gpt-4.1-mini}
/
\textbf{Compressor:} \texttt{gpt-4.1-mini}
}
\\

\midrule

No compression
& 35.7 & 18.14 & 8.55 & 5.07
& 56.1 & 6.45 & 3.72
& 31.2 & 8.31 & 4.79
& 20.6 & 10.64 & 9.18
\\

\midrule

FIFO
& 39.3 & 30.39 & \textbf{6.18} & 5.24
& \textbf{75.4} & \textbf{4.76} & 2.66
& 35.4 & \textbf{5.33} & 4.81
& 9.5 & 8.10 & 7.91
\\

Retrieval
& 14.9 & 40.18 & 7.49 & 5.95
& 36.8 & 7.10 & 4.29
& 8.3 & 7.44 & 6.80
& 0.0 & 7.89 & 6.81
\\

LLMLingua
& 36.3 & 28.41 & 7.24 & 6.65
& 66.7 & 6.96 & 3.84
& 33.3 & 7.05 & 7.60
& 11.1 & 7.62 & 8.47
\\

Prompting
& 35.7 & 24.98 & 6.56 & 4.95
& 64.9 & 5.96 & 2.90
& 27.1 & 6.65 & 5.35
& 15.9 & 6.84 & 6.49
\\

ACON UT
& 42.3 & 22.46 & 6.51 & 5.48
& 64.9 & 5.87 & 2.62
& 37.5 & 7.18 & 5.22
& \textbf{25.4} & 7.18 & 8.25
\\

ACON UTCO
& 32.7 & 24.27 & 6.99 & 4.97
& 57.9 & 7.50 & 2.77
& 33.3 & 8.45 & 4.99
& 9.5 & \textbf{6.95} & 6.97
\\

\rowcolor{blue!10}
FOCUS
& \textbf{44.0} & \textbf{17.90} & 6.71 & \textbf{3.81}
& 66.7 & 6.17 & \textbf{2.46}
& \textbf{41.7} & 6.24 & \textbf{2.82}
& \textbf{25.4} & 7.56 & \textbf{5.78}
\\

\bottomrule
\end{tabular}
\end{adjustbox}

\caption{Detailed results on AppWorld across all difficult categories using \texttt{gpt-4.1-mini} as both the agent and compressor. \methodname{} improves task success while reducing context usage and dependency. We report the defensive configuration of \methodname{} above.}
\label{tab:appworld_gpt-4.1_mini}
\end{table*}
\begin{table*}[ht]
\centering

\footnotesize
\setlength{\tabcolsep}{1.8pt}
\renewcommand{\arraystretch}{0.9}

\begin{adjustbox}{max width=\textwidth}
\begin{tabular}{lccccccccccccc}
\toprule

\textbf{Method}
& \multicolumn{4}{c}{Average}
& \multicolumn{3}{c}{Level 1 (1-app)}
& \multicolumn{3}{c}{Level 2 (2-app)}
& \multicolumn{3}{c}{Level 3 (3-app)} \\

\cmidrule(lr){2-5}
\cmidrule(lr){6-8}
\cmidrule(lr){9-11}
\cmidrule(lr){12-14}

& \textbf{Acc $\uparrow$}
& \textbf{Steps $\downarrow$}
& \textbf{Peak $\downarrow$}
& \textbf{Dep $\downarrow$}

& \textbf{Acc $\uparrow$}
& \textbf{Peak $\downarrow$}
& \textbf{Dep $\downarrow$}

& \textbf{Acc $\uparrow$}
& \textbf{Peak $\downarrow$}
& \textbf{Dep $\downarrow$}

& \textbf{Acc $\uparrow$}
& \textbf{Peak $\downarrow$}
& \textbf{Dep $\downarrow$}
\\

\midrule

\multicolumn{14}{c}{
\textbf{Agent:} \texttt{gpt-4.1-mini}
/
\textbf{Compressor:} \texttt{gpt-4.1-mini}
} \\

\midrule

No Compression
& 72.6 & 11.96 & 7.36 & 3.92
& \textbf{88.1} & 6.66 & 4.29
& 68.2 & 4.97 & 1.01
& 54.8 & 9.02 & 5.40 \\

\midrule

FIFO
& 65.3 & 10.91 & 4.03 & 1.46
& 83.3 & 4.10 & 0.78
& 59.1 & 3.69 & 0.96
& 45.2 & \textbf{4.19} & 2.03 \\

Retrieval
& 67.4 & 14.46 & 4.55 & 2.74
& 85.7 & 5.85 & 5.86
& 59.1 & \textbf{3.47} & 0.87
& 48.4 & 4.59 & 2.45 \\

LLMLingua
& 67.4 & 11.59 & 4.90 & 2.18
& 87.2 & 4.31 & 3.87
& 59.1 & 4.58 & 0.92
& 48.4 & 5.34 & 2.17 \\

Prompting
& 71.6 & 11.78 & 4.93 & 3.10
& 85.7 & 4.73 & 4.75
& 72.7 & 4.40 & \textbf{0.86}
& 51.6 & 5.32 & 3.06 \\

ACON
& \textbf{73.7} & 12.41 & 4.82 & 1.96
& \textbf{88.1} & 4.12 & \textbf{0.83}
& 68.2 & 4.39 & \textbf{0.86}
& 58.1 & 5.37 & 3.07 \\

\rowcolor{blue!10}

\rowcolor{blue!10}
FOCUS
& \textbf{73.7} & \textbf{9.70} & \textbf{4.00} & \textbf{1.50}
& 76.2 & \textbf{3.61} & 1.54
& \textbf{81.8} & 5.38 & 0.93
& \textbf{64.5} & 5.42 & \textbf{1.70} \\

\bottomrule
\end{tabular}
\end{adjustbox}

\caption{Results on OfficeBench using \texttt{gpt-4.1-mini} as both the agent and compressor. \methodname{} reliably drives up task accuracy while maintaining the peak token usage and lowest overall context dependency along with FIFO. We report the defensive configuration of \methodname{} above.}

\label{tab:officebench_mini}
\end{table*}
\begin{table*}[h]
\centering

\footnotesize
\setlength{\tabcolsep}{1.8pt}
\renewcommand{\arraystretch}{0.9}

\begin{adjustbox}{max width=\textwidth}
\begin{tabular}{lccccccccccccc}
\toprule

\textbf{Method}
& \multicolumn{4}{c}{Average}
& \multicolumn{3}{c}{Easy}
& \multicolumn{3}{c}{Medium}
& \multicolumn{3}{c}{Hard} \\

\cmidrule(lr){2-5}
\cmidrule(lr){6-8}
\cmidrule(lr){9-11}
\cmidrule(lr){12-14}

& \textbf{Acc $\uparrow$}
& \textbf{Steps $\downarrow$}
& \textbf{Peak $\downarrow$}
& \textbf{Dep $\downarrow$}

& \textbf{Acc $\uparrow$}
& \textbf{Peak $\downarrow$}
& \textbf{Dep $\downarrow$}

& \textbf{Acc $\uparrow$}
& \textbf{Peak $\downarrow$}
& \textbf{Dep $\downarrow$}

& \textbf{Acc $\uparrow$}
& \textbf{Peak $\downarrow$}
& \textbf{Dep $\downarrow$}
\\

\midrule

\multicolumn{14}{c}{
\textbf{Agent:} \texttt{gpt-4.1}
}
\\

\midrule

Prompting (\texttt{gpt-4.1-mini})
& 39.3 & 23.6 & 7.03 & 5.19
& 64.9 & 6.64 & 3.17
& 35.4 & 7.63 & 5.42
& 19.1 & 6.93 & 6.84 \\

ACON (\texttt{gpt-4.1-mini})
& 47.6 & 21.5 & 7.25 & 5.24
& 75.4 & 6.75 & 2.84
& 35.4 & 7.25 & 5.36
& 31.8 & 7.70 & 7.32 \\


ACON (\texttt{Qwen3-14B})
& 50.0 & 21.7 & 6.83 & 4.80
& 79.0 & \textbf{6.42} & 2.54
& 50.0 & 6.87 & 4.89
& 23.8 & 7.17 & 6.79 \\

ACON (\texttt{Qwen3-8B})
& 47.0 & 21.6 & 6.98 & 4.76
& 71.9 & 6.64 & 2.93
& 37.5 & 7.24 & 4.67
& 31.8 & 7.09 & 6.48 \\

ACON (\texttt{Phi-4})
& 44.6 & 21.2 & 7.24 & 4.76
& 68.4 & 7.33 & 2.75
& 39.6 & 7.12 & 4.16
& 27.0 & 7.26 & 7.04 \\

\rowcolor{blue!10}
\methodname{} (\texttt{gpt-4.1-mini})
& 58.3 & 18.0 & 6.66 & 2.45
& 73.7 & 6.79 & \textbf{2.16}
& 60.4 & 6.17 & \textbf{2.13}
& 42.9 & \textbf{6.92} & \textbf{2.96} \\

\rowcolor{blue!10}
\methodname{} (\texttt{Qwen3-14B})
& 61.3 & \textbf{18.2} & 7.26 & 3.00
& 78.9 & 6.70 & 2.30
& 58.3 & 6.95 & 2.60
& 47.6 & 8.01 & 3.91 \\

\rowcolor{blue!10}
\methodname{} (\texttt{Qwen3-8B})
& \textbf{65.5} & 18.4 & 7.44 & 3.00
& 80.7 & 7.50 & 2.31
& \textbf{66.7} & 6.66 & 2.42
& \textbf{50.8} & 7.98 & 4.04 \\

\rowcolor{blue!10}
\methodname{} (\texttt{Phi-4})
&61.3 &	18.6 & \textbf{6.70} & \textbf{2.62}
&\textbf{84.2}&	6.70	&2.46	
&58.3	&\textbf{6.04}	&\textbf{2.21}
&42.9	&7.20	&3.07 \\
\bottomrule
\end{tabular}
\end{adjustbox}

\caption{Performance on AppWorld using \texttt{gpt-4.1} as the main agent paired with various open-weight models as compressor. We report the defensive configuration. \methodname{} consistently outperforms the baselines increasing average task success (up to 65.5\%) while maintaining minimized context dependencies across all difficulty splits.}

\label{tab:appworld_open_models}
\end{table*}
\begin{table}[t]
\centering

\footnotesize
\setlength{\tabcolsep}{1.5pt}
\renewcommand{\arraystretch}{0.9}

\begin{adjustbox}{max width=\columnwidth}
\begin{tabular}{lcccc}
\toprule
\textbf{Method} & \textbf{Acc $\uparrow$} & \textbf{Steps $\downarrow$} & \textbf{Peak $\downarrow$} & \textbf{Dep $\downarrow$} \\
\midrule

\multicolumn{5}{c}{
\textbf{Agent:} \texttt{gpt-4.1} /
\textbf{Comp:} \texttt{gpt-4.1}
} \\
\midrule

No Compression & 34.8 & \textbf{8.9} & 16.79 & 3.275 \\
\rowcolor{blue!10}
\methodname{} & \textbf{39.3} & 9.1 & \textbf{12.21} & \textbf{2.198} \\
\bottomrule
\end{tabular}
\end{adjustbox}
\caption{Evaluation results on the WebVoyager subset ($135$ tasks, text-only, \texttt{gpt-4.1} backbone). \methodname{} improves accuracy by $+4.5$ points while reducing mean peak tokens by $27\%$ and dependency by $33\%$.}
\label{tab:webvoyager}
\end{table}
\begin{table}[t]
\centering

\footnotesize
\setlength{\tabcolsep}{1.5pt}
\renewcommand{\arraystretch}{0.9}

\begin{adjustbox}{max width=\columnwidth}
\begin{tabular}{lcccccccc}
\toprule
& \multicolumn{4}{c}{\textbf{Retail} ($114$)} & \multicolumn{4}{c}{\textbf{Airline} ($50$)} \\
\cmidrule(lr){2-5} \cmidrule(lr){6-9}
\textbf{Method} & \textbf{Acc $\uparrow$} & \textbf{Steps $\downarrow$} & \textbf{Peak $\downarrow$} & \textbf{Dep $\downarrow$}
              & \textbf{Acc $\uparrow$} & \textbf{Steps $\downarrow$} & \textbf{Peak $\downarrow$} & \textbf{Dep $\downarrow$} \\
\midrule

\multicolumn{9}{c}{
\textbf{Agent:} \texttt{gpt-4.1} /
\textbf{Comp:} \texttt{gpt-4.1}
} \\
\midrule

No Compression & 67.5 & 12.7 & 3.44 & 0.834 & 44.0 & 10.5 & 2.68 & 0.775 \\
\rowcolor{blue!10}
\methodname{} & \textbf{74.6} & \textbf{11.4} & \textbf{3.25} & 0.835 & \textbf{46.0} & \textbf{9.9} & \textbf{2.64} & \textbf{0.774} \\
\bottomrule
\end{tabular}
\end{adjustbox}
\caption{Evaluation results on $\tau^2$-bench. \methodname{} improves accuracy by $+7.1$ points on retail and $+2.0$ on airline. Peak tokens reported in $10^3$ tokens; dependency in $10^6$.}
\label{tab:tau2}
\end{table}

\subsection{Related Work}

\paragraph{Memory-augmented agents.}
A complementary line of work manages long-horizon context through external memory modules rather than by compressing the working trajectory in place. MemGPT~\citep{packer2023memgpt} draws on operating-system virtual memory: it
maintains a tiered hierarchy and lets the agent page information between an in-context main memory and an external store via explicit function calls, retrieving content on demand. Mem0~\citep{chhikara2025mem0} extracts, consolidates,
and retrieves salient facts from a conversation into an external (optionally graph-structured) memory, and is evaluated primarily on long multi-session
dialogue~\citep{maharana2024evaluating}. These systems share our high-level goal of efficient context
management, but differ from \methodname{} along three axes. \emph{(i) Mechanism:}
they add an external store and a retrieval step, reintroducing an indexing and relevance-matching component, whereas \methodname{} performs no retrieval and
introduces no external state, it selects, in place, which thought--action--observation
spans of the live trajectory to retain. \emph{(ii) Retention signal:} their
salience/retrieval signal is typically semantic relevance to a query, well suited
to conversational recall; \methodname{} instead scores each span by its
\emph{forward-looking causal utility} for the agent's future decisions, which is
better matched to tool-using, state-dependent agentic tasks where an
observation's value lies in its effect on subsequent actions rather than its
topical similarity. \emph{(iii) Setting:} MemGPT and Mem0 target long-form
question answering and multi-session chat, while \methodname{} targets
interactive multi-application agents (AppWorld, OfficeBench) where the history is
a growing sequence of executions and environment responses. The two directions
are largely orthogonal and could be combined: an external memory could archive
spans that \methodname{} prunes, allowing later recall. We leave such a hybrid to
future work.

\paragraph{Comparison with Recent Agent Context Compression Methods.}

Recent agent context compression methods such as ACON~\citep{kang2025acon} and PAACE~\citep{yuksel2025paace} rely on offline optimization (guideline refinement or evolutionary prompt search) followed by distillation into
fine-tuned student models, which fixes their compression policy to a training distribution. In contrast, FOCUS is entirely training-free and adapts per trajectory at test time, selecting thought--action--observation spans verbatim
based on self-generated counterfactual utility rather than a learned or externally supplied plan. These distinctions make FOCUS compressor-agnostic and plan-free, applying out-of-the-box to any backbone or domain. 
We summarize the key differences between the methods across training regime, context operation, retention signal, external-plan dependence, and compressor choice in Table~\ref{tab:focus_comparison}.

\begin{table*}[t]
\centering
\footnotesize
\setlength{\tabcolsep}{6pt}
\renewcommand{\arraystretch}{1.15}

\begin{tabularx}{\textwidth}{>{\bfseries}p{0.18\textwidth}XXX}
\toprule
\rowcolor{gray!12}
\textbf{Axis} &
\centering\textbf{ACON} &
\centering\textbf{PAACE} &
\centering\arraybackslash\textbf{FOCUS (Ours)}\\
\midrule

Training regime &
Offline guideline refinement via failure analysis; distilled into student models &
Offline evolutionary prompt search; distilled into SLM using large corpus &
Fully training-free; adapts online for every trajectory \\

\rowcolor{gray!5}
Context operation &
Generative rewriting / summarization of observations and history &
Learned rewriting and summarization &
Verbatim selection of thought--action--observation spans \\

Retention signal &
Optimized natural-language guideline &
Next-$k$-task relevance from an externally supplied plan &
Self-generated counterfactual utility via draft-model rollouts \\

\rowcolor{gray!5}
External plan &
Not required &
Required (benchmark decomposition or separate planner) &
Not required \\

\rowcolor{gray!5}
Compressor &
Distilled student &
Distilled student &
Any off-the-shelf small model (gpt-4.1-mini, Qwen3-8B/14B, Phi-4) \\

\bottomrule
\end{tabularx}

\caption{Comparison of ACON, PAACE, and \methodname{} across key design dimensions. Unlike prior methods, \methodname{} is training-free, compressor-agnostic, and plan-free.}
\label{tab:focus_comparison}
\end{table*}

\subsection{Additional Results on WebVoyager and $\tau^2-$Bench}
\label{app:webvoyager}
To further validate the generality of \methodname{} beyond long-horizon tool-use and QA benchmarks, we conduct additional experiments on two additional benchmarks. \textbf{WebVoyager}~\citep{he2024webvoyager} subset ($135$ tasks, text-only setting) with \texttt{gpt-4.1} as the backbone for both the main agent and the compressor. As shown in Table~\ref{tab:webvoyager}, \methodname{} improves accuracy by $+4.5$ points over the no-compression baseline while simultaneously reducing mean peak tokens by $27\%$ and cumulative dependency by $33\%$. These results confirm that utility-driven span selection remains effective in interactive web-agent settings, where observations (rendered page text, DOM snippets) are long and highly redundant. On $\tau^2-$Bench ~\citep{barres2025tau}, a dual-control conversational benchmark with diffuse, cross-turn dependencies \methodname{} improves task success by $+7.1$ (Table~\ref{tab:tau2}) points on retail and $+2.0$ on airline while maintaining context metrics. This shows that decision-preserving compression not only reduces context cost but can actively improve agent reliability in long, loosely-structured dialogues, where pruning distractor spans sharpens the model's focus on the constraints that determine success.

\subsection{Latency Breakdown}
Table \ref{tab:latency_appendix} provides the latency, token, and cost breakdown for \methodname{} on OfficeBench.
\begin{table}[t]
\centering
\footnotesize
\setlength{\tabcolsep}{4pt}
\renewcommand{\arraystretch}{0.95}
\begin{adjustbox}{max width=\columnwidth}
\begin{tabular}{lccc}
\toprule
\textbf{Metric} & \textbf{No Comp.} & \textbf{FOCUS (seq.)} & \textbf{FOCUS (par.)} \\
\midrule
Total wall clock (s)          & 4480  & 3817  & 3466 \\
Latency per task (s)          & 47.2  & 40.2  & \textbf{36.5} \\
\midrule
Agent input tokens (M)        & 4.364 & 2.584 & 2.671 \\
Agent output tokens (M)       & 0.092 & 0.075 & 0.077 \\
Agent requests                & 1014  & 859   & 872 \\
Draft input tokens (M)        & 0.000 & 0.375 & 0.398 \\
Draft output tokens (M)       & 0.000 & 0.024 & 0.026 \\
Draft requests                & 0     & 249   & 267 \\
\midrule
Agent cost (\$)               & 9.47  & 5.77  & 5.96 \\
Draft cost (\$)               & 0.00  & 0.19  & 0.20 \\
\textbf{Total cost (\$)}      & 9.47  & \textbf{5.96} & 6.16 \\
\midrule
Compression events            & 0     & 83    & 89 \\
Mean compression latency (ms) & --    & 5286  & \textbf{2145} \\
\bottomrule
\end{tabular}
\end{adjustbox}
\caption{Full latency, token, and cost breakdown for \methodname{} on OfficeBench ($95$ subtasks, \texttt{gpt-4.1} agent, \texttt{gpt-4.1-mini} draft, $N{=}3$). The draft model adds only ${\sim}3\%$ to total cost, while parallelizing the $N$ rollouts cuts per-event compression latency by ${\sim}2.5\times$ at essentially identical token cost.}
\label{tab:latency_appendix}
\end{table}

\subsection{Token Statistics}
Table \ref{tab:appworld_token_stats} presents the token consumption and total API cost on AppWorld.
\begin{table}[t]
\centering
\footnotesize
\setlength{\tabcolsep}{5pt}
\renewcommand{\arraystretch}{1.05}
\begin{adjustbox}{max width=\columnwidth}
\begin{tabular}{lcccccc}
\toprule
\multirow{2}{*}{\textbf{Method}}
& \multicolumn{2}{c}{\textbf{Agent tokens (M)}}
& \multicolumn{2}{c}{\textbf{Draft tokens (M)}}
& \multirow{2}{*}{\textbf{Total tok.\ (M)}}
& \multirow{2}{*}{\textbf{Cost (\$)}} \\
\cmidrule(lr){2-3}\cmidrule(lr){4-5}
& \textbf{In} & \textbf{Out} & \textbf{In} & \textbf{Out} & & \\
\midrule
No Compression            & 19.53 & 0.27 & --   & --   & 19.80 & 41.18 \\
\methodname{} (\texttt{gpt-4.1} draft)      & 16.50 & 0.25 & 2.23 & 0.30 & 19.28 & 41.92 \\
\rowcolor{blue!10}
\methodname{} (\texttt{gpt-4.1-mini} draft) & 17.27 & 0.25 & 2.39 & 0.20 & 20.11 & \textbf{37.82} \\
\bottomrule
\end{tabular}
\end{adjustbox}
\caption{Token consumption and total API cost on AppWorld ($168$ tasks, \texttt{gpt-4.1} main agent). \methodname{} reduces the \emph{agent}'s own input tokens (the dominant cost driver) from $19.53$M to $16.50$--$17.27$M by clearing unrequired context; with a small \texttt{gpt-4.1-mini} draft the added draft tokens are cheap enough that total API cost falls below the no-compression baseline ($41.18\rightarrow37.82$). Costs use \$2.00/\$8.00 per $10^6$ input/output tokens for \texttt{gpt-4.1} and \$0.40/\$1.60 for \texttt{gpt-4.1-mini}.}
\label{tab:appworld_token_stats}
\end{table}

\subsection{Benchmarks}
\label{app:benchmarks}
We evaluate our method on five diverse agentic benchmarks spanning API, tool use, question-answering, web interaction and multi turn dialogue. We adopt the dataset split settings used by ACON \cite{kang2025acon} for OfficeBench, AppWorld and 8-QA. Because \methodname{} operates as a training-free, test-time compression framework, we rely exclusively on the test splits for evaluation and do not utilize the training splits. We also furnish relevant statistics (related to Steps, Context, and Tokens) about the benchmarks in Table \ref{tab:long_horizon_stats}.

\paragraph{AppWorld} \citep{trivedi2024appworld} serves as our primary evaluation environment. It provides a high-fidelity execution simulation bridging nine everyday applications (e.g., Gmail, Spotify, Venmo) via 457 APIs, populated with realistic simulated users, explicitly testing the limits of long-horizon productivity agent reasoning. We report our metrics solely on the 168 tasks comprising the \texttt{test\_normal} split. 

\paragraph{OfficeBench}\citep{wang2024officebench} assesses office automation capabilities across applications such as Word, Excel, Email, and Calendar. It categorizes task difficulty intuitively by the number of concurrent applications an agent must coordinate (1-app, 2-app, or 3-app variants). Following the ACON configuration, we report metrics solely on the test split.

\paragraph{8-Objective QA}\citep{zhou2025mem1,kwiatkowski2019natural} adapts standard multi-hop question-answering into a deep-research scenario. Rather than aggregating evidence for a single answer, the agent is presented with eight distinct questions simultaneously and is burdened with constructing a unified, correct final response covering all eight queries. Following the ACON settings (with questions drawn from NaturalQuestions \cite{kwiatkowski2019natural}), the benchmark provides 100 train and 100 test tasks. As with the other environments, we evaluate \methodname{} exclusively on the 100 test tasks.

\paragraph{WebVoyager} \citep{he2024webvoyager} is an end-to-end web-agent benchmark in which an agent completes user instructions by interacting with real-world websites, comprising tasks compiled from $15$ popular real-world websites (e.g., Amazon, Apple, GitHub, Google Maps, arXiv). We adopt the \emph{text-only} setting, in which observations are provided as the textual rendering of the page rather than screenshots, making it a demanding long-horizon environment where page-derived observations are lengthy and highly redundant across steps.

\paragraph{$\tau^2$-bench} \citep{barres2025tau} evaluates conversational agents in a dual-control environment. In contrast to single-control settings where the user is a passive information provider, $\tau^2$-bench models customer-support domains in which both the agent and a simulated user can act on a shared, dynamic state, testing both an agent's reasoning and its ability to communicate with and guide the user through actions on that state. Tasks are produced by a compositional generator that programmatically assembles diverse, verifiable tasks from atomic components, and the environment is coupled with a tool-constrained user simulator; success is measured against verifiable target world states.

\begin{table}[t]
\centering
\footnotesize
\setlength{\tabcolsep}{3pt}
\renewcommand{\arraystretch}{0.9}

\begin{adjustbox}{max width=\columnwidth}
\begin{tabular}{llcccc}
\toprule
\textbf{Benchmark} & \textbf{Metric} & \textbf{Mean} & \textbf{Median} & \textbf{p90} & \textbf{Max} \\
\midrule
\multirow{4}{*}{\shortstack[l]{\textbf{AppWorld}\\{}}}
 & Steps             & 16.9    & 15     & 27      & 41 \\
 & Peak Context      & 10{,}244 & 9{,}278 & 15{,}979 & 27{,}791 \\
 & Cumulative Tokens & 114{,}482 & 87{,}934 & 225{,}490 & 452{,}396 \\
 & Dependency        & 6.36M   & 4.47M  & 12.21M  & 33.22M \\
\midrule
\multirow{4}{*}{\shortstack[l]{\textbf{OfficeBench}\\{}}}
 & Steps             & 10.7    & 9      & 18      & 46 \\
 & Peak Context      & 5{,}409  & 4{,}359 & 9{,}800  & 20{,}749 \\
 & Cumulative Tokens & 38{,}798 & 20{,}023 & 87{,}972 & 464{,}910 \\
 & Dependency        & 2.33M   & 0.94M  & 6.13M   & 31.01M \\
\midrule
\multirow{4}{*}{\shortstack[l]{\textbf{WebVoyager}\\{}}}
 & Steps             & 8.9     & 9      & 15      & 15 \\
 & Peak Context      & 16{,}753 & 11{,}625 & 46{,}072 & 86{,}895 \\
 & Cumulative Tokens & 85{,}207 & 58{,}073 & 197{,}600 & 462{,}356 \\
 & Dependency        & 3.27M   & 2.22M  & 8.57M   & 16.72M \\
\bottomrule
\end{tabular}
\end{adjustbox}

\caption{Trajectory statistics of the \emph{uncompressed} \texttt{gpt-4.1} baseline across three agentic benchmarks. Peak context and cumulative token load grow to tens/hundreds of thousands of tokens with heavy right tails (p90/Max), and dependency reaches tens of millions, quantifying the long-horizon nature that motivates decision-preserving compression. Token counts use \texttt{cl100k\_base}.}
\label{tab:long_horizon_stats}
\end{table}

\subsection{Evaluation Metrics}
\label{app:metrics}

Following ACON~\cite{kang2025acon}, we report following metrics to jointly assess task performance and context efficiency:

\paragraph{Steps.} The average number of agent interaction steps (action--observation exchanges) per task. Fewer steps indicate more efficient task completion.

\paragraph{Peak Tokens.} The maximum number of input tokens observed in any single generation step throughout the agent's trajectory, excluding the static system prompt. Formally, letting $n_i^{(t)}$ denote the input token count (excluding system) at step $t$:
\[
\text{Peak} = \max_{t \in [T]} \; n_i^{(t)}.
\]
This metric serves as a proxy for inference-time memory requirements and reflects the worst-case context length the model must process. We report values in units of $10^3$ tokens.

\paragraph{Dependency (Dep).} The cumulative computational cost incurred by action generation across the full trajectory. At each step $t$, given $n_i^{(t)}$ input tokens (excluding system) and $n_o^{(t)}$ output tokens, dependency is computed as:
\[
\text{Dep} = \sum_{t \in [T]} \frac{(n_i^{(t)} + 2\,n_o^{(t)}) \times n_o^{(t)}}{2}.
\]
We report values in units of $10^6$.

\subsection{Implementation Details}
\label{app:implementation_details}

We outline the primary implementation configurations and hyperparameters used for evaluating \methodname{} across our selected benchmarks.

\paragraph{Inference Details.}
All experiments invoking \texttt{gpt-4.1} and \texttt{gpt-4.1-mini} models \citep{openai2025gpt41} are executed via Azure OpenAI endpoints. For the main agent generation, we set temperature $0.0$ with seed $42$. Conversely, when generating the Monte Carlo rollouts via the draft model, we set the temperature to $0.7$ to encourage a diverse set of stochastic plan sketches. We use \texttt{tiktoken} (cl100k\_base) for token counting. For open-weight draft model experiments, we use HuggingFace \texttt{transformers} with \texttt{Qwen3-8B/14B and Phi}.

\paragraph{Hyperparameters.}
The default \methodname{} configuration relies on $N=3$ draft-model plan rollouts per compression step. The context budget threshold ($\delta$) which dictates when the compression mechanism activates is configured according to the complexity and average sequence length of the tasks in each benchmark. Specifically, we set $\delta = 4096$ tokens for AppWorld tasks due to their long horizons and API response payloads. For OfficeBench and 8-Objective QA, we adopt a threshold of $\delta = 2048$ tokens.

\paragraph{API Cost Estimation.}
For the cost analysis across different model families (as highlighted in Section~\ref{sec:draft_models}), we calculate exact API pricing using standard rates. For the frontier models, the costs per million tokens are formulated as: \texttt{gpt-4.1}: \$2.00 / 1M input tokens, \$8.00 / 1M output tokens. \texttt{gpt-4.1-mini}: \$0.40 / 1M input tokens, \$1.60 / 1M output tokens. For the open-weight draft models (e.g., \texttt{Qwen3-14B}, \texttt{Phi-4}), pricing estimates are standardized using the competitive aggregate rates sourced from \textsl{OpenRouter} (\url{https://openrouter.ai/}).

\paragraph{Draft Model Prompt.}\label{draft_model_prompt}
For completeness we reproduce the exact prompt used to elicit plan sketches and dependency citations. Each compression step issues $N$ stochastic draft calls with the following system prompt (defensive / dual-objective variant):

\begin{tcolorbox}[promptbox]
\ttfamily\small

You are a planning assistant performing Dual-Objective Defensive Drafting. Given a partially completed task and the executor's workspace trace, your goal is twofold:

1. OPTIMISTIC PLANNING: Generate a high-level plan sketch for completion. For each step, cite the exact historical spans you depend on.

Format: Step: <description> | Depends on: [s\_X, s\_Y]

2. PESSIMISTIC VERIFICATION: Review all remaining, un-cited spans in the trace. If discarding an un-cited span would cause the executing agent to blindly repeat a mistake or lose causal state, you MUST rescue it.

Format: Rescued Spans: [s\_A, s\_B] | Reason: <risk if deleted>
\end{tcolorbox}

\paragraph{Rollout Example.}\label{rollout_example}
We provide a draft rollout example on AppWorld with \texttt{gpt-4.1} as both agent and draft generated at a compression point on a specific AppWorld task. The user message provides the original task and the working trace, where each historical span $s_i$ is rendered compactly as

\texttt{[s\_i] | Thought: \ldots | Code: \ldots | Observation:}

\begin{tcolorbox}[promptbox]
\ttfamily\small

OPTIMISTIC PLANNING

Step 1: Identify the correct work schedule note for today and fetch its content to determine available work time. | Depends on: [s\_18]

Step 2: Calculate the available time for tasks by analyzing today's work schedule and leftover tasks. | Depends on: [s\_11, s\_18]

Step 3: Review incomplete tasks in "Today's Goal" and estimate how many can be completed. | Depends on: [s\_11, s\_18]

Step 4: Select the maximum number of incomplete Inbox tasks that fit the schedule. | Depends on: [s\_14, s\_18]

Step 5: Move selected tasks from Inbox to "Today's Goal" via create\_task, keeping details identical. | Depends on: [s\_10, s\_14]

PESSIMISTIC VERIFICATION

Rescued Spans: [s\_6, s\_7, s\_10, s\_17] | Reason: These contain the critical authentication flow and access tokens for Todoist and SimpleNote. Losing them would prevent further API access. All other uncited spans are exploratory or failed attempts.

\end{tcolorbox}

\subsection{Qualitative Analysis of Failure and Success Cases}
\label{app:qualitative_case_study}

To characterize when decision-preserving compression helps and when it can hurt, we conduct a counterfactual case study on AppWorld. For representative tasks we compare trajectories with and without compression and, for each span, inspect whether it was retained or discarded and what downstream effect this had. This analysis surfaces a failure mode and clarifies the complementary regime in which forward-looking compression is beneficial.

\paragraph{Failure mode: set-generalized plans undervalue per-entity spans.}
\methodname{} scores a span by how often a forward plan sketch depends on it. This estimator is reliable when future decisions reference specific prior observations, but it becomes miscalibrated when the remaining task is set-valued, i.e., the same operation must be applied to many independent entities. Consider the following task, which requires unioning two entity lists (phone roommates and existing friends) and acting on each member.

\begin{tcolorbox}[promptbox]
\small
\textbf{AppWorld Task \texttt{ff58e36} (Set-Valued Goal)}\\[2pt]
\textit{``Add all my friends and roommates as friends on Venmo, if they are not already.''}
\end{tcolorbox}

\noindent
In its plan sketch, the draft model expresses the remaining work as a single set-generalized step: ``for each contact, search and befriend on Venmo'' and therefore attributes forward utility to the \emph{procedure} while assigning low utility to the earlier spans that had enumerated the individual contacts. Those enumeration spans are consequently pruned. With one entity no longer present in the working context, the agent befriends four of the five required users and omits the fifth. The failure is thus directly attributable to a specific discarded span (the contact enumeration). We observe the same mechanism across sibling instances of this task template.

\paragraph{Condition for miscalibration.}
This failure arises specifically when a task requires carrying many independent, per-entity records across applications (e.g., befriend $N$ contacts, pay $N$ recipients, update $N$ rows) and the draft model produces a set-generalized plan (``for each item \ldots''), it under-weights the spans that hold the individual set members, and pruning them drops entities silently. As $N$ grows, the set-generalized plan increasingly under-represents individual members.

\paragraph{Where forward-looking compression helps.}
The complementary case is equally informative: the same pruning that is risky under set-valued goals is beneficial when long histories induce attention dilution. We analyze two such tasks.

\begin{tcolorbox}[promptbox]
\small
\textbf{AppWorld Task \texttt{b9c5c9a\_2} (Long-History Reconciliation)}\\[2pt]
\textit{``I have invited some of my friends to a reunion party via phone messages. I have made a CSV to track who is coming or not in \texttt{\textasciitilde/documents/personal\_stuff/} in my file system. Please update RSVPs in it as per their latest replies.''}
\end{tcolorbox}

\begin{tcolorbox}[promptbox]
\small
\textbf{AppWorld Task \texttt{9016950\_1} (Entity Deduplication)}\\[2pt]
\textit{``I need my parents to have a Venmo account. Last time I checked none had one. Make an account for whoever does not have it yet, using their email address and \texttt{A\}2Gm4r} as password. Then send them a phone text message: `I have created a venmo account for you. Please activate it, you should have received an email for it. I've set your password to be \texttt{A\}2Gm4r}. Change it soon too.'\,''}
\end{tcolorbox}

\noindent
In task \texttt{b9c5c9a\_2} (updating an RSVP spreadsheet from many phone-message threads), the uncompressed agent operates over a very large accumulated context ($\sim$20K-token peak) and produces an incomplete file, dropping several respondents; \methodname{}, by retaining only decision-relevant spans ($\sim$6K-token peak), reconstructs the correct file. Similarly, in \texttt{9016950\_1} the uncompressed agent loses track of which entities were already handled and over-includes recipients, whereas \methodname{} preserves the relevant state and acts on the correct entity. These cases illustrate that decision-preserving compression is not merely lossless bookkeeping: by removing distracting history that the base agent would otherwise mishandle, it can \emph{improve} the base agent's decisions on long-horizon tasks.

\paragraph{Takeaway.}
Compression is beneficial when it removes history that would otherwise dilute the agent's attention, and risky when the draft model undervalues a span's contribution to the future trajectory. Defensive verification (Sec.~\ref{method: defensive_verification}) is precisely designed to recover such spans, mitigating this effect by explicitly preserving evidence that a broad plan (in this case, a set-generalized plan) would otherwise discard.

\subsection{Potential Risks}
\label{app:potential_risks}

We identify limited direct societal risks from this work, as FOCUS is a compression utility layer rather than an autonomous decision-making system. However, two indirect risks merit acknowledgment. First, by enabling longer and cheaper agent execution, our method could lower the barrier to deploying under-supervised autonomous agents in sensitive domains (e.g., financial transactions, email management), where errors may have real-world consequences. Second, aggressive context compression could in principle remove safety-relevant spans (e.g., policy violation warnings or user-consent confirmations), potentially enabling an agent to bypass safeguards. We mitigate this via the defensive verification mechanism, which explicitly preserves failure and constraint signals; nevertheless, practitioners should validate retention behaviour in safety-critical deployments.

\subsection{Licenses.}
All benchmarks used in this work are publicly available under permissive licenses: AppWorld~\citep{trivedi2024appworld}, OfficeBench~\citep{wang2024officebench} and Natural Questions~\citep{kwiatkowski2019natural} are released under the Apache~2.0 license.

\end{document}